\documentclass[10pt,twocolumn,letterpaper]{article}

\usepackage{cvpr}
\usepackage{times}
\usepackage{epsfig}
\usepackage{graphicx}
\usepackage{amsmath}
\usepackage{amssymb}

\usepackage{amsthm,paralist}
\usepackage{graphicx,subcaption}
\usepackage{relsize}
\usepackage{multirow}
\usepackage{subcaption}
\usepackage{algorithm}
\usepackage{algorithmic}
\usepackage{booktabs}
\usepackage{float}
\usepackage{enumitem}

\def\argmax{\mathrm{argmax}}
\def\argmin{\mathrm{argmin}}

\def\argmax{\mathrm{argmax}}
\def\argmin{\mathrm{argmin}}

\newcommand{\N}{\mathbb{N}}
\newcommand{\E}{\mathbb{E}}

\newcommand{\R}{\mathbb{R}}

\newcommand{\relu}{\text{ReLU}}

\newcommand{\herm}{\text{hermite}}
\renewcommand{\slash}{/\penalty\exhyphenpenalty\hspace{0pt}}

\newtheorem{theorem}{Theorem}

\newtheorem{lemma}[theorem]{Lemma}

\usepackage[pagebackref=true,breaklinks=true,letterpaper=true,colorlinks,bookmarks=false]{hyperref}

\cvprfinalcopy 

\def\cvprPaperID{2236} 

\ifcvprfinal\fi
\begin{document}

\title{Generating Accurate Pseudo-labels in Semi-Supervised Learning and Avoiding Overconfident Predictions via Hermite Polynomial Activations \thanks{Please direct correspondence to Lokhande, Ravi, Singh. \newline
Accepted at 2020 IEEE/CVF Conference on Computer Vision and Pattern Recognition (CVPR). }}

\author{Vishnu Suresh Lokhande\\
	{\tt\small lokhande@cs.wisc.edu}
	\and
	Songwong Tasneeyapant\\
	{\tt\small tasneeyapant@wisc.edu}
	\and
	Abhay Venkatesh\\
	{\tt\small abhay.venkatesh@gmail.com}
	\and
	Sathya N. Ravi\\
	{\tt\small sathya@uic.edu}
	\and
	Vikas Singh\\
	{\tt\small vsingh@biostat.wisc.edu}
}

\maketitle

\begin{abstract}
	    Rectified Linear Units (ReLUs) are
	    among the most widely used activation function
	    in a broad variety of tasks in vision.
	    Recent theoretical results suggest that despite their
	    excellent practical performance,
	    in various cases, a substitution with
	    basis expansions (e.g., polynomials) can yield
	    significant benefits from both the optimization
	    and generalization perspective. Unfortunately,
	    the existing results remain limited
	    to networks with a couple of layers, and the practical
	    viability of these results is not yet known.
	    Motivated by some of these results,
	    we explore the use of Hermite polynomial expansions as a substitute for
	    ReLUs in deep networks. While our experiments with supervised learning
	    do not provide a clear verdict, we find that this strategy
	    offers considerable benefits in semi-supervised learning (SSL) / transductive learning
	    settings. We carefully develop
	    this idea and show how the use of Hermite polynomials based
	    activations can yield improvements in pseudo-label accuracies and
	    sizable financial savings (due to concurrent runtime benefits).
	    Further, we show via theoretical analysis,
	    that the networks (with Hermite activations)
	    offer robustness to noise and other attractive mathematical properties. Code is available on \href{https://github.com/lokhande-vishnu/DeepHermites}{//GitHub}.
\end{abstract}

\begin{figure*}[!t]
	\begin{center}
		\centerline{\includegraphics[width=1.75\columnwidth]{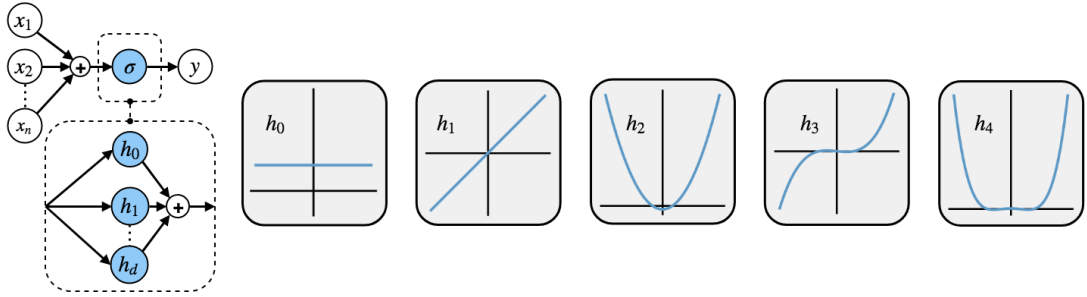}}
		\caption{\label{hermite_relu} \footnotesize \textbf{Hermite Polynomials as Activations} ({\bf leftmost}): Incorporating Hermite Polynomials as an activation function in a single hidden unit one hidden layer network. ({\bf middle}): The functional form of the first 5 hermites are shown in the right.}
	\end{center}
	\vskip -0.4in
\end{figure*}

\section{Introduction}
\label{sec:intro}
Analyzing the optimization or the loss landscape of deep neural networks has emerged 
as a promising means to understand the behavior and properties of various neural network architectures \cite{choromanska2015loss}.
One reason 
is that insights into how the loss function behaves geometrically is closely tied to the types of optimization schemes that may be
needed \cite{zaheer2018adaptive}, 
why specific ideas work whereas others do not \cite{santurkar2018does}, and how or whether the corresponding model may generalize to unseen data \cite{zhang2016understanding}.
Notice that we must leverage such ``alternative'' strategies as a window into these models' behavior because in deep learning, 
most models of interest are highly non-linear  and 
non-convex. As a result, one finds that extending mature ideas, which work well for analysis in classical settings (e.g., linear models), is quite a bit harder for most deep architectures
of interest.  

{\bf Why study activation functions?} While there are many ways we may study the optimization landscape of deep models, a productive line of recent results \cite{ge2017learning} proposes 
analyzing the landscape (i.e., the behavior of the loss as a function of the network parameters) via the {\bf activation functions} of the neural network.
This makes a lot of sense because the activation function is one critical place where we introduce non-linearity into the network, and their omission significantly 
simplifies any analysis. 
Activation functions greatly influence the functional space which a neural network can represent \cite{kileel2019expressive}, often the first 
step in a more formal study of the model's behavior. 
For example, universal finite-sample expressivity of certain architectures has been shown by fixing the activations to 
be ReLU functions \cite{hardt2016identity}. 
In other words, if we use ReLU as activations, such an architecture 
can be shown to represent any function if the model has more parameters than the sample size.
The scope of such work is not just theoretical -- for instance, the above results were used 
to derive a much simpler architecture consisting of only residual convolutional layers and ReLU activations. 
Further, the estimation scheme needed was also much simpler and required no batch normalization, dropout, or max pooling. In summary, the choice of activations
enabled understanding the loss landscape and enabled simplifications. 

{\bf More on activations functions and loss landscape.}
The authors in \cite{venturi2018spurious} showed that conditions that prevent presence of spurious valleys on the loss landscape
can be identified, via the use of smooth activations. Independently, 
\cite{soltanolkotabi2018theoretical} demonstrated that the use of quadratic activation functions
enables efficient localization of global minima in certain classes of deep models.
Closely related to smooth and quadratic activations, the use of polynomial non-linearity as an activation
has also been studied in the last few years. For example, 
\cite{poggio2017theory} studied deep ReLU networks by using a polynomial non-linearity as an approximation for ReLU: this enabled a much
cleaner analysis of the empirical risk landscape.
More recently, \cite{kileel2019expressive} analyzed the functional space of the networks with the help of polynomial activation
based on techniques from algebraic geometry.
Earlier, for a one hidden layer network, 
\cite{ge2017learning} investigated optimizing the population risk of the loss using
stochastic gradient descent and showed that one could
avoid spurious local minima by utilizing an orthogonal basis expansion for ReLUs.
A key takeaway from this work is that the optimization would behave better if the
landscape was nicely behaved --- this is enabled via the basis expansion. 
The foregoing results and analyses, especially the use of {\bf basis expansion}, is interesting and 
a starting point for the development described here.
A common feature of the body of work summarized above is that they rely on networks with polynomial activations (polynomial networks) to analyze the loss landscape.
This choice helps make the mathematical exploration easier.

{\bf Where is the gap in the literature?} Despite the interesting theoretical results summarized above, relatively less is known whether such
a strategy or its variants are a good idea
for the architectures in computer vision and broadly, in AI, today.
We can, in fact,
ask a more practically focused question: are there specific tasks in computer vision where such a strategy
offers strong empirical advantages? This is precisely the gap this paper
is designed to address. 

The {\bf main contributions} include:
\begin{inparaenum}[\bfseries (a)]
	\item We describe mechanisms via which activation functions based on Hermite polynomials can be utilized within
	deep networks instead of ReLUs, with only minor changes to the architecture.
	\item We present evidence showing that while these adjustments are not 
	significantly advantageous in supervised learning, our scheme {\em does} yield sizable advantages in semi-supervised
	learning speeding up convergence. Therefore, it offers clear benefits in compute time (and cost) needed to attain
	a certain pseudo-label accuracy, which has direct cost implications.
	\item We give technical results analyzing the mathematical behavior of such activations, specifically, robustness results showing how the activation mitigates overconfident predictions for (out of distribution) samples. 
\end{inparaenum}


\section{Brief Review of Hermite polynomials}
\label{primer}
We will use an expansion based on Hermite polynomials as a substitute for ReLU activations.
To describe the construction clearly, we briefly review the basic properties. 

{\bf Hermite polynomials} are a class of orthogonal polynomials which are appealing both theoretically
as well as in practice, e.g., radio communications \cite{boyd1984asymptotic}.
Here, we will use Hermite polynomials \cite{rasiah1997modelling} defined as, 
\begin{align}
	\label{eq:hermite}
	H_i(x) = (-1)^i e^{x^2} \frac{d^i}{d x^i}e^{-x^2} \text{, $i > 0$;  }  	H_0(x) = 1
\end{align}
In particular, we use normalized Hermite polynomials which are  given as $h_i = \frac{H_i}{\sqrt{i!}}$. Hermite polynomials are often used in the analysis of algorithms for nonconvex optimization problems \cite{mossel2005noise}. While there are various mathematical properties associated with Hermites, we will now discuss the most important property for our purposes.

{\bf Hermite polynomials as bases.} Classical results in functional analysis
show that the (countably infinite) set $\{H_i\}_{i=0}^{d=\infty}$ defined in \eqref{eq:hermite}
can be used as bases to represent smooth functions \cite{rudin2006real}.  Formally, let $L^2(\R, e^{-x^2/2})$ denote the set of integrable functions w.r.t. the Gaussian measure, 
$$L^2(\R, e^{-x^2/2}) = \{ f : \int_{-\infty}^{+\infty} f(x)^2 e^{-x^2/2} dx < \infty \},$$ It turns out that $L^2(\R, e^{-x^2/2})$ is a
Hilbert space with the inner product defined as follows (see \cite{ge2017learning} for more details)
\begin{align}
\langle f, g \rangle &= \E_{x \sim \N(0, 1)} [f(x) g(x)] 
\end{align}
The normalized Hermite polynomials form an {\em orthonormal basis} in $L^2(\R, e^{-x^2/2})$ in the sense that $\langle h_i, h_j \rangle = \delta_{ij}$. Here $\delta_{ij}$ is the Kronecker delta function and
$h_i, h_j$ are any two normalized Hermite polynomials.

Recently, the
authors in \cite{ge2017learning} showed that the lower order terms in the
Hermite polynomial series expansion of ReLU have a different rate of convergence
on the optimization landscape than the higher order terms for one hidden layer networks. 
This leads to a natural {\bf question}: are these properties useful to accelerate the performance
of gradient-based methods for deep networks as well?

{\bf Hermite Polynomials as Activations.} 
Let $x = (x_1, \ldots, x_n)$ be an input to a neuron in
a neural network and $y$ be the output. Let $w = (w_1, w_2, \ldots, w_n)$ be the weights associated with the neuron. Let $\sigma$ be the non-linear activation applied to $w^Tx$.
Often we set $\sigma=\relu$. Here, we investigate the
scenario where
$\sigma(x) = \sum_{i=0}^d c_i h_i (x)$, denoted as $\sigma_{\herm}$, were $h_i$'s are as defined previously
and $c_i$'s are {\bf trainable parameters}.
As \cite{ge2017learning} suggests, we initialized the parameters $c_i's$ associated with hermites to be $c_i = \hat\sigma_i$, where $\hat\sigma_i = \langle \relu, h_i \rangle$.
Hermite polynomials as activations on a single layer neural network with a single hidden unit can be seen in Figure~\ref{hermite_relu}.

\section{Sanity checks: What do we gain or lose?}
\label{resnets}

Replacing ReLUs with other activation functions reviewed in Section~\ref{sec:intro} has been variously
attempted in the literature already, for supervised learning. While improvements have been reported in specific
settings, ReLUs continue to be relatively competitive. Therefore, it seems that we should
not  expect improvements in what are toy supervised learning experiments. However, as we describe below,
the experiments yield useful insights regarding settings where Hermites will be particularly useful. 

\begin{figure}[!b]
	\begin{center}
		\centerline{
			\includegraphics[width=0.5\columnwidth]{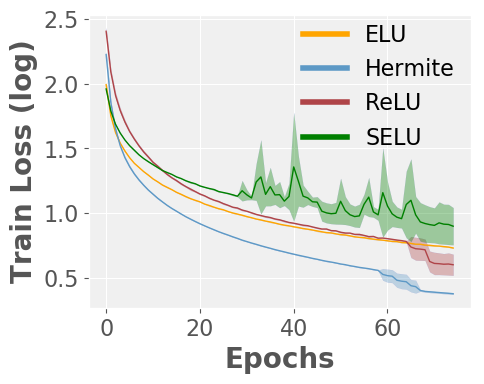}
			\includegraphics[width=0.5\columnwidth]{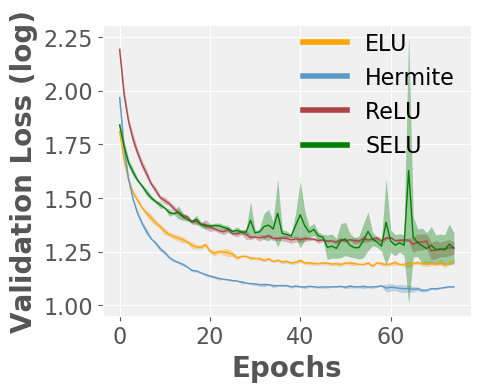}}
		\caption{\label{supp_ae_activ} \footnotesize \textbf{Other Activation Functions.} We observe an increase in the rate of convergence of the validation loss for Hermites  when compared to other activations such as ELU, ReLU and SeLU. A standard auto-encoder model as outlined in \cite{zaheer2018adaptive} is used to test on MNIST dataset.}
	\end{center}
	\vskip -0.2in
\end{figure} 

{\bf A) Two-layer architectures.} We start with the CIFAR10 dataset and a simple two-layer network (more details in the appendix $\S$~\ref{sec_app_shallow}). Basically, 
we ask 
if 
this modification will lead to better or worse performance in accuracy and runtime (over ReLU activation)
with all other parameters (e.g., learning rates) fixed. 
{\bf Positives:} 
Here, we observed {\em faster} training loss convergence (over the initial epochs) compared to ReLU in general.
{\bf Negatives:} 
When we assessed the performance as a function of the number of Hermite polynomials $d$,
we see a tradeoff where the speeds first improve with increasing $d$ and then worsen when $d$ is as high as $8$,
indicating that too many bases, especially for relatively shallow architectures are not ideal. 

{\bf B) Deep Autoencoders.}
Encouraged by the shallow network experiments, we moved to a standard benchmark for neural network optimization runtimes:
the deep autoencoders setup from \cite{zaheer2018adaptive}.
We replicate the experiments reported in \cite{zaheer2018adaptive} for the MNIST dataset: we use $4$ Hermite polynomials
as an activation instead of the original sigmoid.
{\bf Positives:} We see from Table~\ref{tab:AE} that Hermite activations not only 
converge faster but also achieve lower test loss.
We also evaluated how activation functions such as ReLU, SeLU or ELU perform in this setting. As shown in Figure~\ref{supp_ae_activ}, we find that Hermites still offer runtime improvements here with a  comparable validation loss. We find that SeLU does not seem to be as stable as the other activation functions. Therefore, in the remainder of this paper, we study Hermite activations compared to ReLUs as a baseline. Of course, 
it is quite possible that some other activations may provide slightly better 
performance compared to ReLUs in specific settings. But the choice above simplifies the presentation of our results and is consistent 
with the general message of our work: while research on polynomial activations has so far remained largely theoretical, we intend to show 
that they are easily deployable and offer various practical benefits. 

\begin{table}[!t]
	\centering
	\resizebox{\columnwidth}{!}{%
		\begin{tabular}{ccccc}
			Method & LR & $\epsilon$  & Train Loss & Test Loss \\ \hline\hline
			Sigmoid-Adam \cite{zaheer2018adaptive} &    $10^{-3}$  &  $10^{-8}$  & $2.97 \pm 0.06$ & $7.91 \pm 0.14$  \\ 
			Sigmoid-Adam \cite{zaheer2018adaptive} &    $10^{-3}$ &   $10^{-3}$  &  $1.90 \pm 0.08$ & $4.42 \pm 0.29$  \\
			Hermite-Adam (Ours)&    $10^{-3}$ &   $10^{-8}$  &  \textbf{1.89} $\pm$ \textbf{0.01}  &  \textbf{3.16} $\pm$ \textbf{0.02}  \\
			\hline\hline
		\end{tabular}%
	}
	\caption{\label{tab:AE} \footnotesize {\bf Deep autoencoders with Hermite Activations give lower test loss} Results from our implementation following directions from  \cite{zaheer2018adaptive}.}
	\vskip -0.2in
\end{table}

\begin{figure}[!b]
	\begin{center}
		\centerline{\includegraphics[width=0.8\columnwidth]{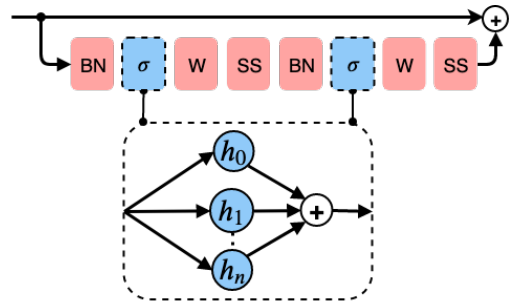}}
		\caption{\label{resnets_preact} \footnotesize \textbf{Hermite Polynomials as Activations in ResNets}. We introduce softsign function to
			handle the numerical issues from the unbounded nature of Hermite polynomials.  $W$ denotes the weight, $BN$ denotes batch normalization, $\sigma$ is the
			Hermite activation and $SS$ is the softsign function.}
	\end{center}
	\vskip -0.2in
\end{figure}

\begin{figure*}[!t]
	\begin{subfigure}[b]{0.75\linewidth}
		\includegraphics[width=\linewidth]{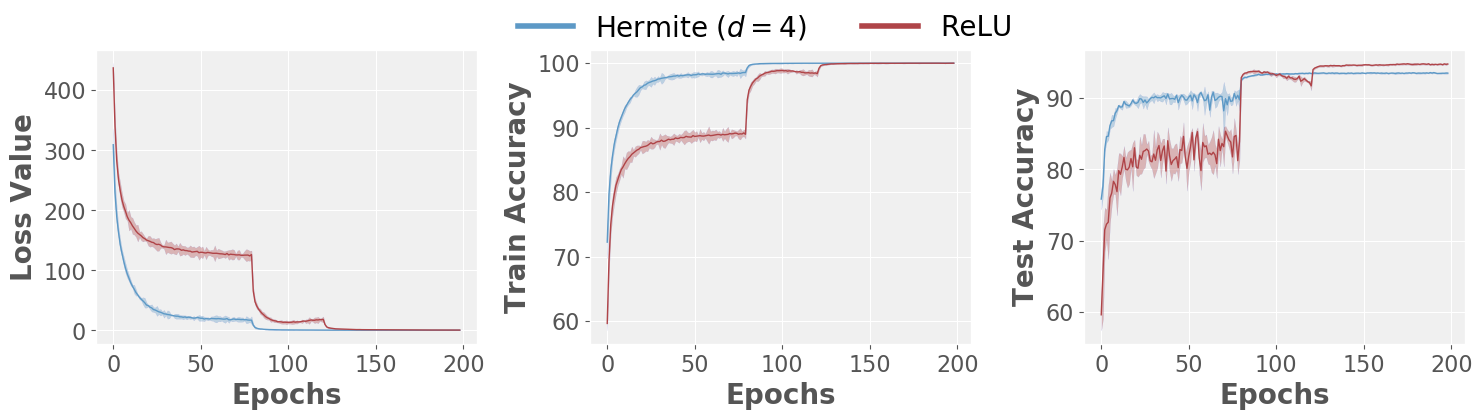}
		\caption{Loss and Accuracy curves}
	\end{subfigure}%
	\begin{subfigure}[b]{0.25\linewidth}
		\includegraphics[width=\linewidth]{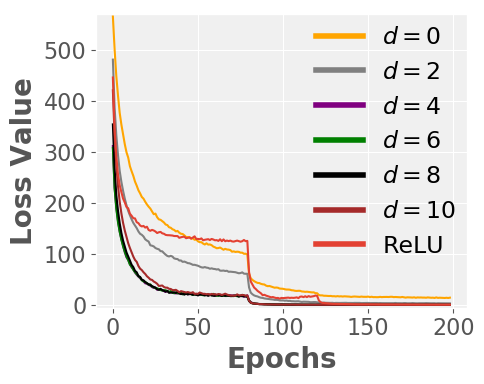}
		\caption{Effects of changing $d$}
		\label{resnets_change_d}
	\end{subfigure}%
	\caption{\label{resnets_loss_train} \footnotesize \textbf{Hermites vs. ReLUs on ResNet18.} (a) Hermites provide faster convergence of train loss and train accuracies than ReLUs. Hermites have faster convergence in test accuracies over the initial epochs but ReLU has the higher test accuracy at the end of training. (b) As we increase the number of hermite polynomials, the speed of loss convergence increases until $d=6$ and then it starts to reduce. $d\ge1$ performs better than $d=0$ where only softsign is used as an activation. The close overlap between $d=4$ to $d=10$ suggests that tuning for $d$ is not very expensive. }
		\vskip -0.2in
\end{figure*}

{\bf Adjustments needed for deeper architectures?}
When implemented naively, Hermite activations do {\bf not} work well for deeper architectures directly, which
may be a reason that they have not been carefully explored so far. 
Basically, with no adjustments, we encounter a number of numerical issues that
are not easy to fix. In fact,
\cite{glorot2011deep} explicitly notes that higher-order polynomials tend to
make the activations unbounded making the training unstable. 
Fortunately, a trick mentioned in \cite{bergstra2009quadratic} in the
context of quadratic functions, addresses the problem. 
The solution is to add a softsign function in  \eqref{eq:softsign}, 
which has a form similar to tanh however, it approaches its maximal (and minimal) value slower than tanh.
\begin{align}
  \label{eq:softsign}
  \text{Softsign}(x) = \frac{x}{1+|x|}
\end{align}

{\bf C) ResNet18.}
With the above modification in hand, we can use our activation function 
within Resnet18 using Preactivation Blocks \cite{he2016deep,he2016identity}.
In the preactivation block, we found that having the second softsign function {\em after}
the weight layer is useful. 
The slightly modified preactivation block of ResNets is shown in Figure~\ref{resnets_preact}. 
We train ResNets with Hermite activations on
CIFAR10 to assess the general behavior of our substitution.
We use SGD as an optimizer and follow the data augmentation techniques and
the learning rate schedules as in \cite{he2016deep,chollet2015keras}.
We perform cross-validation for the hyper-parameters.
After training, we obtain the loss curve and the  training set accuracies for
Hermite activations and ReLUs (see Fig. \ref{resnets_loss_train}).
{\bf Positives:}
Figure~\ref{resnets_loss_train} shows that the loss values and trainset accuracies converge at a much faster rate for Hermites than ReLUs.
While the test set accuracy at the {\em final} epoch is higher for ReLU than Hermites,
the test set accuracies for networks using Hermite activations
{\em make much quicker progress in the initial epochs}.
These results hold even when varying the learning rates of ReLU networks (see appendix $\S$~\ref{sec_app_resnet}).
%
%
{\bf Negatives:} We also tune the choice of the number of Hermite polynomials $d$ for $d \in \{0, 2, 4, 6, 8, 10\}$. The setting  $d = 0$ is the case where
we only use a softsign as an activation without any Hermite activations.
Figure~\ref{resnets_change_d} shows the results of this experiment.
From the plot, we observe a trend similar to the two-layer network above,
where the convergence speeds first improves and then reduces as we increase the
number of Hermite polynomials. The close overlap of the trend lines between $d=4$ to $d=10$ suggest that tuning for $d$ is not very expensive. Hence, in all our experiments we mostly set $d=4$. The setting $d=0$ performs worse than when $d$ is at least one, suggesting that the performance benefits are due to Hermite activations with softsign (and not
the softsign function on its own).

{\bf Interesting take away from experiments?}
Let us consider the negative results first where we find that a large number of bases is not useful. This makes sense where
some flexibility in terms of trainable parameters is useful, but too much flexibility is harmful. Therefore,
we simply need to set $d$ to a small (but not too small) constant. On the other hand, the positive
results from our simple experiments above suggest that networks with Hermite activations make rapid
progress in the early to mid epoch stages -- an {\bf early riser} property -- and the performance gap
becomes small later on. This provides two potential strategies. If desired, we could design
a hybrid optimization scheme that exploits this behavior. One difficulty is
that initializing a ReLU based network with weights learned for a network with
Hermite activations (and retraining) may partly offset the benefits from
the quicker progress made in the early epochs. What will
be more compelling is to utilize the Hermite activations end to end, but
identify scenarios where this ``early riser'' property is critical and
directly influences the final goals or outcomes of the task.
It turns out that recent developments in semi-supervised learning satisfy
exactly these conditions where leveraging the early riser property is
immensely beneficial.

\section{Semi-Supervised Learning (SSL)}
\label{application}

{\bf SSL, Transductive learning and pseudo-labeling.}
We consider the class of algorithms that solve SSL by generating pseudo-labels
for the unlabelled data. These \textit{Pseudo-label (PL) based SSL} methods
serve as a way to execute transductive learning \cite{iscen2019label,shi2018transductive}, where
label inference on the {\bf given test dataset} is
more important for the user as compared to using the trained model later,
on other unseen samples from
other data sources.
Stand-alone PL based methods can be used 
for generic SSL problems also -- where the generation of pseudo-labels are merely a means
to an end. The literature suggests that
in general they perform competitively, but some 
specialized SSL methods may provide some marginal benefits.
In the other direction, general (non PL based)
SSL methods can also be used for obtaining pseudo-labels.
But PL based SSL generates
pseudo-labels concurrently with (in fact, to facilitate) estimating the parameters of the model.
So, if the eventual goal is to obtain labels for unlabeled data at hand,
these methods
are a good fit -- they explicitly assign (potentially) high accurate labels for the unlabeled
data at the end of training.

\begin{algorithm}[!t]
	\caption{}
	\label{alg:saas}
	\begin{algorithmic}
		\STATE {\em Input:} Labeled data $(x_i,y_i)$, unlabeled data $(z_i)$, number of classes $k$, \#-outer (inner) epochs $M_O (M_I)$, loss function $L = L_{CE}(x_i,y_i) + L_{CE}(z_i,y_i) + Reg_E(z_i,y_i)$, learning rates $\eta_w,\eta_P^p,\eta_P^d$. Initial Pseudo-labels for unlabeled data chosen as: $y_i=e_i$ with probability $1/k$ where $e_i$ is the one-hot vector at $i-$th coordinate.
		\FOR{$O = 0, 1, 2, ...,M_O$}
		\STATE Reinitialize the network parameters $w^0 $ 
		\STATE  $\mathit{\Delta} P_u = 0$
		\FOR{$I = 0, 1, 2, ...,M_I$}
		\STATE (Primal) SGD Step on $w$: $w^{t+1} \gets w^{t} -\eta_w \nabla L$
		\STATE (Primal) SGD Step on $\Delta P_u$: $\mathit{\Delta} P_u \gets \mathit{\Delta} P_u - \eta_P^p \nabla L$
		\ENDFOR
		\STATE (Dual) SGD Step on $P_u$: $P_u \gets P_u - \eta_P^d\mathit{\Delta} P_u$
		\ENDFOR
		\STATE {\em Output:} Classification model $w$.
	\end{algorithmic}
\end{algorithm}

\begin{figure}[!t]
	\begin{center}
		\centerline{\includegraphics[width=0.75\columnwidth]{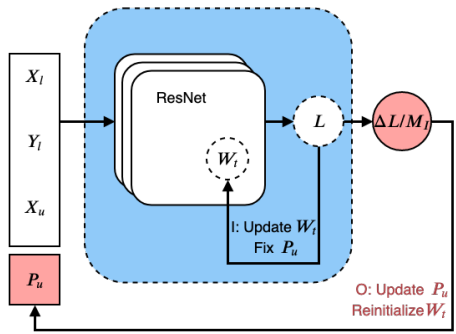}}
		\caption{\label{saas_flowchart} \footnotesize \textbf{SaaS framework illustration. }SaaS runs over two loops:
			inner loop denoted by $I$ and outer loop denoted by $O$. }
	\end{center}
	\vskip -0.4in
\end{figure}

{\bf Improved deep classifiers by assigning PL.} 
It is well known that poor (random) labels are harder to fit (train) \cite{zhang2016understanding}, or
equivalently, high quality (pseudo) labels accelerates the training procedure. In other words,
highly accurate labels $\to$ fast training. 
Interestingly, \cite{Cicek_2018_ECCV} showed that the converse statement, i.e., fast training $\to$ accuracy of labels, 
is {\em also} true using empirical evidence, during training.
SaaS ({\bf Speed as a Supervisor}) is used
to {\em estimate} the pseudo-labels using ``speed'' as a surrogate.
That is, for a classifier such as ResNet,
SaaS takes advantage of the fact that the {\em loss decreases at a faster rate for correct
labels} (compared to random labels) during training. Furthermore,
the {\em rate of loss reduction decreases as the percentage of incorrect labels in the dataset increases}. 
{\bf The SaaS framework.} There are two phases in SaaS. In the primal phase (inner loop), SaaS seeks
to find a set of pseudo-labels that decreases the loss function over a {\em small number of epochs}, the most.
In the dual phase (outer loop), the ``pseudo-label'' optimization is carried out
by computing a posterior distribution over the unlabeled data.
Specifically, the algorithm consists of two loops: (i) in the outer loop,
we optimize over the posterior distribution of the pseudo-labels, and (ii) in the inner loop,
we retrain the network with fixed pseudo-labels.
After every inner loop, we reinitialize the network, and compute the rate of change of the loss value
with which the posterior can be computed. A flowchart is shown in Figure \ref{saas_flowchart}.
We can easily
use a ResNet/DenseNet model in the primal phase.
There are two different loss functions that are optimized during training: the cross-entropy loss ($L_{CE}$) over the labeled data,
the unlabeled data and an entropy regularizer ($Reg_E$) on the pseudo-labels.
A pseudocode is provided in Algorithm~\ref{alg:saas}.

{\bf Pseudo-labels with Hermites.} Recall that
networks with Hermite activations manifest the ``early riser'' property. This turns out to be
ideal in the setting described above and next, we show how this property can be exploited for
transductive/semi-supervised learning. 
Intuitively, the early riser property implies that the
training loss decreases at a {\bf much faster rate} in the initial epochs.
This is expected, since the optimization
landscape, when we use Hermites are, by definition, {\em smoother} than ReLUs,
since {\em all} the neurons are {\em always} active during training with probability $1$. 

\begin{table}[!t]
	\centering
	\resizebox{\columnwidth}{!}{%
		\begin{tabular}{c c c c c}
			\multirow{1}{*}{Dataset}   & \# Labeled.             & \# Unlabeled.            &  Augmnt.                      & $M_I$/ $M_O$  \\ \hline\hline
			\multirow{1}{*}{SVHN}      & \multirow{1}{*}{1K} & \multirow{1}{*}{72,257} & A $+$ N              & \multirow{1}{*}{5 $ \slash $ 75}\\\hline  
			\multirow{1}{*}{CIFAR-10}  & \multirow{1}{*}{4K} & \multirow{1}{*}{46,000} & A $+$ N         & \multirow{1}{*}{10 $ \slash $ 135} \\\hline
			\multirow{1}{*}{SmallNORB} & \multirow{1}{*}{1K} & \multirow{1}{*}{23,300} & \multirow{1}{*}{None}      & \multirow{1}{*}{10 $ \slash $ 135} \\\hline
			\multirow{1}{*}{MNIST}     & \multirow{1}{*}{1K} & \multirow{1}{*}{59,000} & \multirow{1}{*}{N} & \multirow{1}{*}{1 $ \slash $ 75}   \\
			\hline\hline
		\end{tabular}%
	}
	\caption{\label{saas_hp_table} \footnotesize \textbf{SSL experimental details.} $A$ and $N$ denote the data augmentation techniques, affine transformation and normalization respectively. $M_I$ and $M_O$ denote the inner and outer epochs respectively.}
	\vskip -0.2in
\end{table}

{\bf Setting up.} For our experiments,
we used a ResNet-18 architecture (with a preactivation block) \cite{he2016identity} architecture to
run SaaS \cite{Cicek_2018_ECCV} with {\bf one crucial difference}: ReLU activations were replaced with Hermite activations.
All other hyperparameters needed are provided in Table \ref{saas_hp_table}.
We will call this version, {\bf Hermite-SaaS}.
To make sure that our findings are broadly applicable, we conducted experiments with four datasets commonly used in the semi-supervised learning literature: SVHN, CIFAR10, SmallNORB, and MNIST.

\begin{figure*}[!t]
	\begin{center}
		\centerline{\includegraphics[width=\linewidth]{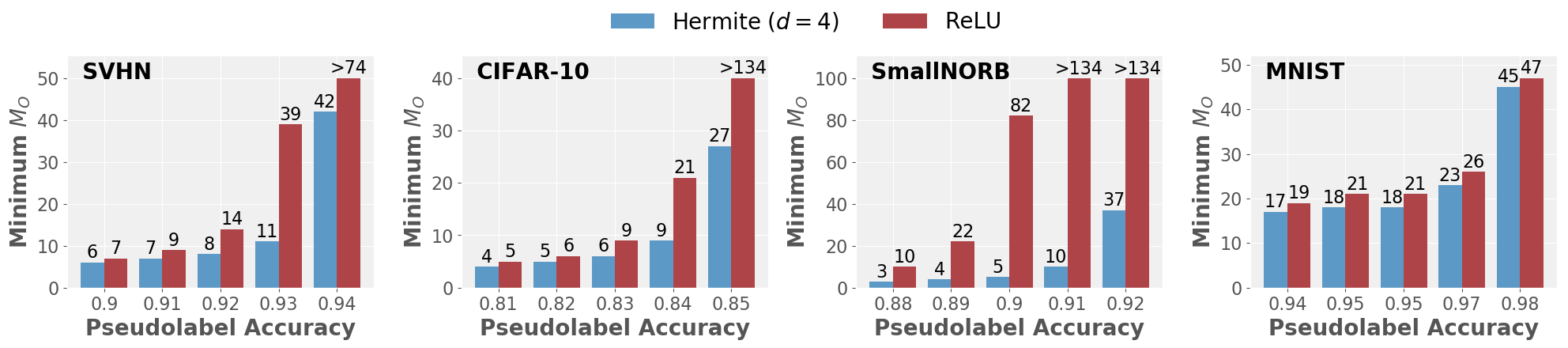}}
		\caption{\label{phase1_epochs_accuracy} \footnotesize \textbf{Hermite-SaaS trains faster.} We plot the number of outer epochs $M_O$ vs. the pseudo-label accuracy across $4$ datasets. We consistently observe that the minimum number of outer epochs $M_O$ to reach a given value of pseudo-label accuracy is always lower for Hermite-SaaS than ReLU-SaaS.}
	\end{center}
	\vskip -0.4in
\end{figure*}

\begin{figure}[!b]
	\begin{center}
		\centerline{\includegraphics[width=\columnwidth]{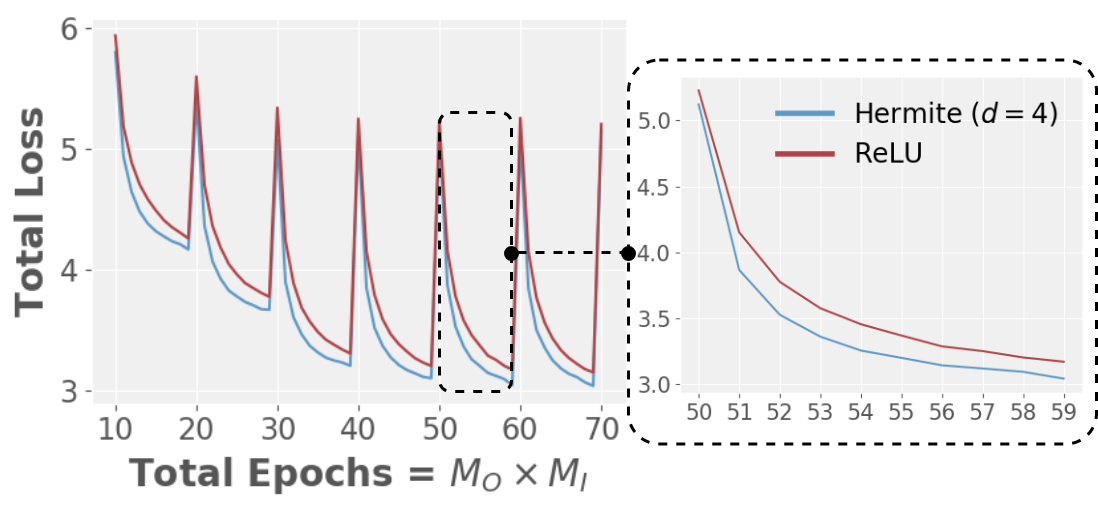}}
		\caption{\label{phase1_cifar_verif} \footnotesize\textbf{Convergence of loss functions in SaaS. }Larger spikes correspond to the end of inner loop and are due to weight reinitialization. Right plot shows that Hermites accelerate the training, ensuring high-quality pseudo-labels.}
	\end{center}
	\vskip -0.2in
\end{figure}

\subsection{Faster Convergence.}
\label{sec:saas_results:fast_convergence}
SaaS tries to identify labels on which the training loss decreases at a faster rate. Our experiments show that
the use of Hermites provides models on which the loss function can be optimized easily,
thus stabilizing the two phase procedure.
Figure \ref{phase1_cifar_verif} shows the result of our experiment on CIFAR10 dataset.
Notice that the periodic jumps in the loss value is expected.
This is because the
big jumps correspond to the termination of an inner epoch, where
the pseudo-labels are updated and weights are reinitialized.
From Figure \ref{phase1_cifar_verif}, we observe that Hermite-SaaS provides
a smoother landscape during training, accelerating the training process overall. 


\subsection{Computational Benefits.}
\label{sec_phase1_aws} 
When the accuracy of pseudo-labels is assessed against the actual labels for SVHN dataset, we obtain
an error rate of \textbf{5.79} over the pseudo-label error rate of 6.22 in \cite{Cicek_2018_ECCV}. This indicates that the quality of pseudo-labels can be significantly
improved using Hermite-SaaS. Moreover, we also find that the number of epochs needed to reach a specific pseudo-label accuracy is also significantly lower. Table~\ref{phase1_aws} shows two common metrics used in classification: (i) ``max gap in accuracy'' (Max $\Delta$) measures the maximum difference in
pseudo-label accuracy over epochs during training; and (ii) ``pseudo-label accuracy'' (Max PL ACC) measures the maximum pseudo-label accuracy attained during the course of training. Both these measures form a proxy to assess the the quality of pseudo-labels. In addition, Figure \ref{phase1_epochs_accuracy} shows that the number of epochs needed to reach a
specific pseudo-label accuracy is {\bf significantly lower} when
using Hermite-SaaS. This behavior for Hermite-SaaS holds over all the four datasets. 

\begin{table}[!b]
	\centering
	\resizebox{\columnwidth}{!}{%
		\begin{tabular}{c c c c c c c c}
			&  &  &     & Time/ & Total      &                                  &                                                  \\ 
			&  &  Max &   Max PL   & Epoch    & Time       & Cost                             & Saved                                          \\ 
			&   & $\Delta$  &   ACC      & (sec)    & (hrs)    & (\$)                             & (\$)                                             \\ \hline\hline
			\multirow{2}{*}{SVHN}     & H & \multirow{2}{*}{6.1\%} &  94.2\% & 470   & 5.6       & \multicolumn{1}{c}{137}       & \multicolumn{1}{c}{\multirow{2}{*}{71}}      \\ 
			
			& R  &  & 93.3\% & 409    & 8.5       & \multicolumn{1}{c}{208}        & \multicolumn{1}{c}{}                            \\ \hline
			\multirow{2}{*}{CIFAR10}   & H & \multirow{2}{*}{3.4\%} & 85.5\% & 348   & 2.7      & \multicolumn{1}{c}{66}          & \multicolumn{1}{c}{\multirow{2}{*}{$\ge$ 213}}   \\ 
			& R &  & 84.4\% & 304   & $\ge$ 11 & \multicolumn{1}{c}{$\ge$ 280} & \multicolumn{1}{c}{}                            \\ \hline
			Small & H & \multirow{2}{*}{5.2\%}  & 92.6\% & 47     & 0.5       & \multicolumn{1}{c}{12}       & \multicolumn{1}{c}{\multirow{2}{*}{$\ge$ 13}} \\ 
			NORB                 & R &  & 90.4\%    & 27   & $\ge$ 1 & \multicolumn{1}{c}{$\ge$ 25}  & \multicolumn{1}{c}{}                            \\ \hline
			\multirow{2}{*}{MNIST}     & H  & \multirow{2}{*}{4.5\%} & 98.2\% & 94    & 1.2       & \multicolumn{1}{c}{29} & \multicolumn{1}{c}{\multirow{2}{*}{-11}}      \\ 
			& R &  &  98.2\%  & 55    & 0.3       & \multicolumn{1}{c}{18}  & \multicolumn{1}{c}{}                            \\ \hline\hline
		\end{tabular}
	}
	\caption{\label{phase1_aws} \footnotesize \textbf{Cost effective and accurate pseudo-labels are generated by Hermite-SaaS.} Expenses when training Hermite-SaaS (H) and ReLU-SaaS (R) on AWS and the performance metrics on pseudo-labels generated. We observe that although Hermite-SaaS takes more time per epoch than ReLU-SaaS, the overall gains are better for Hermite-SaaS.}
\end{table}

\subsection{Financial Savings.}
ReLU takes less time per iteration since in our Hermite-SaaS procedure, we must perform
a few more matrix-vector multiplications.
However, our experimental results indicate that the lower {\em per iteration }time of ReLU is
negligible as compared to the total number of iterations.
To provide a better context for the amount of savings possible using Hermite-SaaS, we performed
an experiment to assess financial savings. 
We use AWS p3.16x large cluster for training.
We calculate the cost (to train a network)
by running the algorithm to reach a minimum level of pseudo-label accuracy, using the default pricing model given by AWS.
We can clearly see from Table \ref{phase1_aws}, that we get significant cost savings if we use Hermite-SaaS. Note that we found cases where ReLU-SaaS could not reach the pseudo-label accuracy that is achieved
by Hermite-SaaS: in these cases, we report a conservative lower bound for cost savings.

\begin{table}[!t]
	\centering
	\resizebox{\columnwidth}!{
		\begin{tabular}{c c c c c c c}
			\multirow{2}{*}{Method} &\textbf{ SVHN} &   Method & & \multicolumn{2}{c}{\textbf{CIFAR10}} & \\ 
			& 1000 & (SSL with PL) & 500 & 1000 & 2000 & 4000 \\
			\cmidrule(lr){1-2}\cmidrule(lr){3-7} \morecmidrules
			\cmidrule(lr){1-2}\cmidrule(lr){3-7}
			VAT+EntMin \cite{miyato2018virtual} & 3.86 & \multirow{2}{*}{TSSDL \cite{shi2018transductive}} &  \multirow{2}{*}{-} &  \multirow{2}{*}{21.13} & \multirow{2}{*}{14.65} & \multirow{2}{*}{10.90} \\ \cmidrule(lr){1-2}
			MT \cite{tarvainen2017mean} & 3.95 & & & & &  \\ \cmidrule(lr){1-2}\cmidrule(lr){3-7}
			TSSDL \cite{shi2018transductive} & 3.80 & \multirow{2}{*}{Label Prop. \cite{laine2016temporal}} &  \multirow{2}{*}{32.4} &  \multirow{2}{*}{22.02} & \multirow{2}{*}{15.66} & \multirow{2}{*}{12.69} \\\cmidrule(lr){1-2}
			TE \cite{laine2016temporal} & 4.42 & & & & &  \\\cmidrule(lr){1-2}\cmidrule(lr){3-7}
			ReLU-SaaS  & 3.82 & ReLU-SaaS & - & - & - & 10.94 \\ \cmidrule(lr){1-2}\cmidrule(lr){3-7}
			Hermite-SaaS & \textbf{3.57 $\pm$ 0.04} & Hermite-SaaS & \textbf{29.25} & \textbf{20.77} & \textbf{14.28} & \textbf{10.65} \\
			\cmidrule(lr){1-2}\cmidrule(lr){3-7} \morecmidrules
			\cmidrule(lr){1-2}\cmidrule(lr){3-7}
		\end{tabular}
	}
	\caption{\label{tab:sota_results} \footnotesize\textbf{Hermite-SaaS generalizes better.} (\textbf{left}): Test-set accuracies on SVHN dataset in comparison to the baselines provided in \cite{Cicek_2018_ECCV}. (\textbf{right}):Test-set accuracies on CIFAR10 dataset in comparison to other pseudo-label based SSL methods.}
	\vskip -0.2in
\end{table}

\subsection{Better Generalization.}
From the work of \cite{zhang2016understanding}, we know that more accurate labels provide better generalization. With
more accurate pseudo-labels from Hermite-SaaS, we expect a standard supervised classifier, trained using the pseudo-labels, to provide better accuracy on an unseen test set. We indeed observe this behavior in our experiments. The left sub-table in Table~\ref{tab:sota_results} shows the performance of Hermite-SaaS on SVHN datast against popular SSL methods reported in \cite{Cicek_2018_ECCV}. We also compare against known PL based SSL methods for CIFAR10 dataset in right sub-table of Table~\ref{tab:sota_results}. We use ResNet18 with a preactivation block for the supervised training phase, although other networks
can be used. 



\begin{figure*}[!t]
	\centering
	\begin{subfigure}[b]{0.3\linewidth}
		\includegraphics[width=\linewidth]{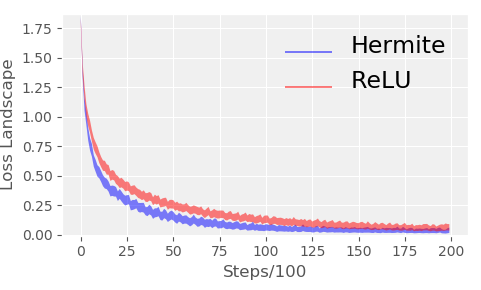}
	\end{subfigure}%
	\begin{subfigure}[b]{0.3\linewidth}
		\includegraphics[width=\linewidth]{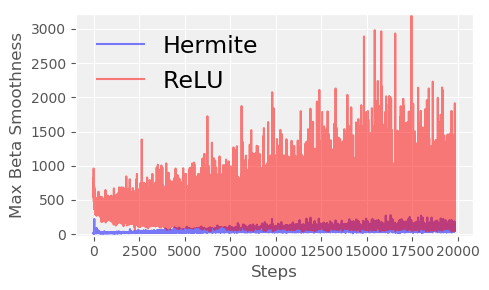}
	\end{subfigure}%
	\begin{subfigure}[b]{0.3\linewidth}
		\includegraphics[width=\linewidth]{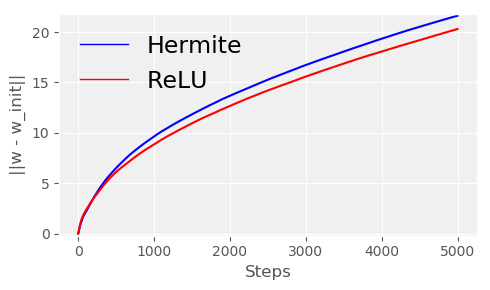}
	\end{subfigure}%
	\caption{\label{fig:verify} \footnotesize \textbf{Hermite generates smoother landscape than ReLU.} (\textbf{left}): Magnitude of loss and its variation is lower for Hermites. (\textbf{middle}): Gradients are more stable on Hermite loss landscape: lower maximum beta smoothness. (\textbf{right}): Hermite networks have a higher effective learning rate .}
	\vskip -0.2in
\end{figure*}

\subsection{Noise Tolerance.}
\label{noise_tolerance}
SSL approaches, especially PL methods, suffer from confirmation bias \cite{tarvainen2017mean} which is when the network starts to predict incorrect labels with higher confidence and resists change in the course of training. A recent result
describes how
ReLU networks \cite{hein2018relu} tend to provide
high confidence predictions even when the test data is different from the training data: this can be used in adversarial
settings \cite{nguyen2015deep}.
In other words, \cite{tarvainen2017mean} and \cite{hein2018relu} taken together hint at the possibility that
ReLUs may, at least, partly contribute to the confirmation bias in PL based SSL methods. 
While evaluating this question comprehensively is an interesting topic on its own,
we present some interesting preliminary results. 
We can assess the question partly by asking whether Hermites-SaaS is more tolerant to noise than ReLU-SaaS: approached
by artificially injecting noise into the labels. 
In this section, we first evaluate this issue on the theoretical side (for one hidden-layer Hermite networks)
where we find that Hermite networks 
do not give (false) high confidence predictions when the test data are different from the train data.
We later demonstrate noise tolerance of Hermite-SaaS.

{\bf{Accurate confidence of predictions in one-hidden layer hermite networks.}}
In order to prove our main result, we first prove a generic pertubation bound in the following lemma. Consider a $2$ layer network with an input layer, hidden layer and
an output layer, each with multiple units.
Denote $f_k(x)$ and $f_l(x)$ to be two output units with $x$ denoting
the input vector. Let $w_j$ be the weight vector between input and the $j^{th}$
and $a_{lk}$ be the weight connecting $l^{th}$ hidden unit to the $k^{th}$ output unit.
\begin{lemma}
	\label{lemma1}
	Consider the output unit of the network, $f_k(x) = \sum_j a_{kj} \sum_{i=0}^d c_i h_i(w_j^Tx)$, 
	where $c_i's$ are the Hermite coefficients and
	$d$ is the maximum degree of the herimite polynomial considered.
	Then, 
	$$\left|f_l(x) - f_k(x)\right| \le C d \alpha \beta$$
	Here, $C$ is a constant proportional to the
$\ell_\infty$ norms of $c_i$'s;
	$$\alpha = \max_{lk} \sum_j \left| a_{lj}-a_{kj}\right| \text {  ;  } \beta = \max(\|w\|_p^d\|x\|_q^d, \|w\|_p\|x\|_q),$$
	such that $1/p+1/q=1$.
\end{lemma}
Now, we use the perturbation bound from Lemma \ref{lemma1} to show the  following result  (proof in the appendix $\S$~\ref{sec_app_proof}) that characterizes the behavior of a network if the test example is ``far'' from examples seen during training. 
\begin{theorem}\label{thm1}
	Let $f_k(x) = \sum_j a_{kj}\sum_{i=0}^{\infty} c_i h_i(w_j^T x)$, be a one-hidden layer network with the sum of infinite series of Hermite polynomials as an activation function. Here, $k=1,2,...,K$ are the different classes.
	Define $w_J = \min w_j^T x$. Let the data $x$ be mean normalized. If $\epsilon > 0$, the
	Hermite coefficients $c_i = (-1)^i$ and
	\begin{align*}
	\|x\| \ge \frac{1}{\|w_J\|} \log{\frac{\alpha}{\log{(1+K\epsilon)}}}
	\end{align*}
	then, we have that the predictions are approximately (uniformly) random. That is,
	\begin{align*}
	\frac{1}{K} - \epsilon \le \frac{e^{f_k(x)}}{\sum_{l=1}^K e^{f_l(x)}} \le \frac{1}{K} + \epsilon \text{  $\forall$ $k \in \{1,2...,K\}$}
	\end{align*}
	\begin{proof} {\em Sketch.} Note that the form of coefficients is important for this theorem. In particular, we use the exponential generating functions expansion of Hermites and exploit the form of normalization due to the softmax layer to provide a lower bound for the confidence of prediction for an arbitrary class. For this event to occur with high probability, we show that the test data has to be at least a certain distance far away from training data.
	\end{proof}
\end{theorem}
As the  data is mean normalized, any test example
$x$ with high $\|x\|$ implies that it is far from the training data.
For large $\|x_{\text{test}}\|$, this result shows that the predictions are
fairly random -- a desirable property -- clearly, we should not predict confidently
if we have not seen such an example earlier.

{\bf{Hermite-SaaS is more noise tolerant.}}
To further assess if Hermite-SaaS reduces confirmation bias, we experiment by injecting noise in the labels provided for SSL training. In particular, we chose the label for a subset of images,
uniformly at random.  We conduct experiments with  $30\%$ label noise levels on
CIFAR10. After estimating the pseudo-labels we trained a model in a supervised manner using Resnet18. 
Our results show that Hermite-SaaS based models  obtain similar or a higher test set accuracy of about $80\%$.
This is encouraging, but we also observe 
faster convergence \textbf{(95 epochs)} compared to a ReLU-SaaS model \textbf{(at least 600 epochs)}.
In other words, Hermite activations
based training yields models/estimators with low variance suggesting 
that they may behave well in the presence of outliers. We also investigate how Hermite activations behave in presence of a standard robust learning method with noisy labels, specifically \cite{tanaka2018joint}. We observe that the Hermites version boosts the
performance of \cite{tanaka2018joint} both in terms of accuracy and rate of convergence (see appendix $\S$~\ref{sec_app_noiseinjection}).


\section{Why Hermites provide faster convergence?}
We discussed how the noise tolerance properties of Hermites help with faster convergence of Hermite-SaaS.
Here, we show how incorporating Hermite activations within a network makes the loss landscape
smoother relative to ReLUs. Smoother landscapes implying faster convergence is not controversial \cite{ghadimi2013stochastic}.
One difference between ReLUs and Hermites is the {\bf nonsmooth} behavior: for ReLU networks,
standard first order methods require $O(1/\epsilon^2)$ (versus $O(1/\epsilon$) iterations for Hermite nets) to find a local minima. This is
the reason why is it not enough to just replace Hermites with (a sufficiently large number of) ReLUs even though ReLU networks are universal approximators. 
We provide three pieces of empirical evidences to support our claim that Hermites provide smoother landscape.

{\bf{(a) Lower  Loss Values.}}
 We examine the loss landscape, along the SGD trajectory, following the directions outlined in \cite{santurkar2018does}. The authors there showed that {\tt \small BN} generates smoother landscapes: 
in Fig.~\ref{fig:verify}, we show that {\tt \small BN+Hermites} generate even smoother landscapes implying much faster training.
In Fig.~\ref{fig:verify}  (left), we plot training loss $L(w - \eta\nabla L)$ for different $\eta$ values for ResNet18 architecture. Hermite network generates a better  loss landscape  (lower magnitude) than ReLU network. 
{\bf{ (b) Maximum Beta smoothness.}}
In Fig.~\ref{fig:verify} (middle), we show the maximum difference in the $\ell_2$ norm of gradients over the distance moved in that direction. Hermite networks have a lower variation of gradient norm change than ReLU networks, indicating that the gradients on Hermite Landscape are more stable implying faster convergence. 
{\bf (c) Higher effective learning rate.}
In Fig.~\ref{fig:verify} (right), we plot deviation of weights from initialization and observe an increase in this metric with Hermites. 

\section{Conclusion}
  In this paper, we
  studied the viability and potential benefits
  of using a finite Hermite polynomial bases as activation functions, as a substitute for
  ReLUs.
  The lower order Hermite polynomials have
  nice mathematical properties from the optimization point of view,
  although little is known
  in terms of their practical applicability to networks with more
  than a few layers. 
  We observed from our experiments that simply
  replacing ReLU with an expansion in terms of
  Hermite polynomials can yield significant computational benefits,
  and we demonstrate the utility of this idea in a
  computationally intensive semi-supervised learning task.
  Under the assumption that the training is being performed on the cloud (published pricing structure), 
  we show sizable financial savings are possible.
  On the mathematical side, we also showed
  that Hermite based networks have nice noise stability properties that appears to be an
  interesting topic to investigate, from the robustness or adversarial angles. Furthermore, since Hermite-nets avoid over-confident predictions on newer test samples, it would be interesting to investigate the benefits of using Hermite-nets to solve (variants of) meta learning problems.

\section*{Acknowledgments}
Research supported by NIH R01 AG062336, NSF CAREER award RI$\#$1252725 and American Family Insurance.
We thank one of the reviewers for pointing out a promising connection to meta-learning that will be pursued in follow-up work.
We are grateful to Glenn Fung for discussions and pointing out a nice paper on smooth SVMs by Yuh-Jye Lee and Olvi Mangasarian \cite{lee2001ssvm}
which dealt with smoothing the ``plus'' function. Rest in peace, Olvi.

{\small
\bibliographystyle{ieee_fullname}
\bibliography{ms.bib}

\begin{thebibliography}{10}\itemsep=-1pt

\bibitem{bergstra2009quadratic}
James Bergstra, Guillaume Desjardins, Pascal Lamblin, and Yoshua Bengio.
\newblock Quadratic polynomials learn better image features.
\newblock {\em Technical report, 1337}, 2009.

\bibitem{boyd1984asymptotic}
John~P Boyd.
\newblock Asymptotic coefficients of hermite function series.
\newblock {\em Journal of Computational Physics}, 54(3):382--410, 1984.

\bibitem{chollet2015keras}
Fran\c{c}ois Chollet et~al.
\newblock Keras.
\newblock \url{https://keras.io}, 2015.

\bibitem{choromanska2015loss}
Anna Choromanska, Mikael Henaff, Michael Mathieu, G{\'e}rard~Ben Arous, and
  Yann LeCun.
\newblock The loss surfaces of multilayer networks.
\newblock In {\em Artificial Intelligence and Statistics}, pages 192--204,
  2015.

\bibitem{Cicek_2018_ECCV}
Safa Cicek, Alhussein Fawzi, and Stefano Soatto.
\newblock Saas: Speed as a supervisor for semi-supervised learning.
\newblock In {\em The European Conference on Computer Vision (ECCV)}, September
  2018.

\bibitem{ge2017learning}
Rong Ge, Jason~D Lee, and Tengyu Ma.
\newblock Learning one-hidden-layer neural networks with landscape design.
\newblock {\em arXiv preprint arXiv:1711.00501}, 2017.

\bibitem{ghadimi2013stochastic}
Saeed Ghadimi and Guanghui Lan.
\newblock Stochastic first-and zeroth-order methods for nonconvex stochastic
  programming.
\newblock {\em SIAM Journal on Optimization}, 23(4):2341--2368, 2013.

\bibitem{glorot2011deep}
Xavier Glorot, Antoine Bordes, and Yoshua Bengio.
\newblock Deep sparse rectifier neural networks.
\newblock In {\em Proceedings of the fourteenth international conference on
  artificial intelligence and statistics}, pages 315--323, 2011.

\bibitem{hardt2016identity}
Moritz Hardt and Tengyu Ma.
\newblock Identity matters in deep learning.
\newblock {\em arXiv preprint arXiv:1611.04231}, 2016.

\bibitem{he2016deep}
Kaiming He, Xiangyu Zhang, Shaoqing Ren, and Jian Sun.
\newblock Deep residual learning for image recognition.
\newblock In {\em Proceedings of the IEEE conference on computer vision and
  pattern recognition}, pages 770--778, 2016.

\bibitem{he2016identity}
Kaiming He, Xiangyu Zhang, Shaoqing Ren, and Jian Sun.
\newblock Identity mappings in deep residual networks.
\newblock In {\em European conference on computer vision}, pages 630--645.
  Springer, 2016.

\bibitem{hein2018relu}
Matthias Hein, Maksym Andriushchenko, and Julian Bitterwolf.
\newblock Why relu networks yield high-confidence predictions far away from the
  training data and how to mitigate the problem.
\newblock {\em arXiv preprint arXiv:1812.05720}, 2018.

\bibitem{huang2017densely}
Gao Huang, Zhuang Liu, Laurens Van Der~Maaten, and Kilian~Q Weinberger.
\newblock Densely connected convolutional networks.
\newblock In {\em Proceedings of the IEEE conference on computer vision and
  pattern recognition}, pages 4700--4708, 2017.

\bibitem{iscen2019label}
Ahmet Iscen, Giorgos Tolias, Yannis Avrithis, and Ondrej Chum.
\newblock Label propagation for deep semi-supervised learning.
\newblock In {\em Proceedings of the IEEE Conference on Computer Vision and
  Pattern Recognition}, pages 5070--5079, 2019.

\bibitem{kileel2019expressive}
Joe Kileel, Matthew Trager, and Joan Bruna.
\newblock On the expressive power of deep polynomial neural networks.
\newblock {\em arXiv preprint arXiv:1905.12207}, 2019.

\bibitem{laine2016temporal}
Samuli Laine and Timo Aila.
\newblock Temporal ensembling for semi-supervised learning.
\newblock {\em arXiv preprint arXiv:1610.02242}, 2016.

\bibitem{le2015simple}
Quoc~V Le, Navdeep Jaitly, and Geoffrey~E Hinton.
\newblock A simple way to initialize recurrent networks of rectified linear
  units.
\newblock {\em arXiv preprint arXiv:1504.00941}, 2015.

\bibitem{lee2001ssvm}
Yuh-Jye Lee and Olvi~L Mangasarian.
\newblock Ssvm: A smooth support vector machine for classification.
\newblock {\em Computational optimization and Applications}, 20(1):5--22, 2001.

\bibitem{miyato2018virtual}
Takeru Miyato, Shin-ichi Maeda, Masanori Koyama, and Shin Ishii.
\newblock Virtual adversarial training: a regularization method for supervised
  and semi-supervised learning.
\newblock {\em IEEE transactions on pattern analysis and machine intelligence},
  41(8):1979--1993, 2018.

\bibitem{mossel2005noise}
Elchanan Mossel, Ryan O'Donnell, and Krzysztof Oleszkiewicz.
\newblock Noise stability of functions with low influences: invariance and
  optimality.
\newblock In {\em 46th Annual IEEE Symposium on Foundations of Computer Science
  (FOCS'05)}, 2005.

\bibitem{nguyen2015deep}
Anh Nguyen, Jason Yosinski, and Jeff Clune.
\newblock Deep neural networks are easily fooled: High confidence predictions
  for unrecognizable images.
\newblock In {\em Proceedings of the IEEE conference on computer vision and
  pattern recognition}, 2015.

\bibitem{poggio2017theory}
Tomaso Poggio and Qianli Liao.
\newblock {\em Theory II: Landscape of the empirical risk in deep learning}.
\newblock PhD thesis, Center for Brains, Minds and Machines (CBMM), arXiv,
  2017.

\bibitem{rasiah1997modelling}
AI Rasiah, R Togneri, and Y Attikiouzel.
\newblock Modelling 1-d signals using hermite basis functions.
\newblock {\em IEE Proceedings-Vision, Image and Signal Processing},
  144(6):345--354, 1997.

\bibitem{rudin2006real}
Walter Rudin.
\newblock {\em Real and complex analysis}.
\newblock Tata McGraw-Hill Education, 2006.

\bibitem{santurkar2018does}
Shibani Santurkar, Dimitris Tsipras, Andrew Ilyas, and Aleksander Madry.
\newblock How does batch normalization help optimization?
\newblock In {\em Advances in Neural Information Processing Systems}, pages
  2483--2493, 2018.

\bibitem{shi2018transductive}
Weiwei Shi, Yihong Gong, Chris Ding, Zhiheng MaXiaoyu~Tao, and Nanning Zheng.
\newblock Transductive semi-supervised deep learning using min-max features.
\newblock In {\em Proceedings of the European Conference on Computer Vision
  (ECCV)}, pages 299--315, 2018.

\bibitem{shin2019trainability}
Yeonjong Shin and George~Em Karniadakis.
\newblock Trainability and data-dependent initialization of over-parameterized
  relu neural networks.
\newblock {\em arXiv preprint arXiv:1907.09696}, 2019.

\bibitem{soltanolkotabi2018theoretical}
Mahdi Soltanolkotabi, Adel Javanmard, and Jason~D Lee.
\newblock Theoretical insights into the optimization landscape of
  over-parameterized shallow neural networks.
\newblock {\em IEEE Transactions on Information Theory}, 65(2):742--769, 2018.

\bibitem{tanaka2018joint}
Daiki Tanaka, Daiki Ikami, Toshihiko Yamasaki, and Kiyoharu Aizawa.
\newblock Joint optimization framework for learning with noisy labels.
\newblock In {\em Proceedings of the IEEE Conference on Computer Vision and
  Pattern Recognition}, pages 5552--5560, 2018.

\bibitem{tarvainen2017mean}
Antti Tarvainen and Harri Valpola.
\newblock Mean teachers are better role models: Weight-averaged consistency
  targets improve semi-supervised deep learning results.
\newblock {\em arXiv preprint arXiv:1703.01780}, 2017.

\bibitem{venturi2018spurious}
Luca Venturi, Afonso~S Bandeira, and Joan Bruna.
\newblock Spurious valleys in two-layer neural network optimization landscapes.
\newblock {\em arXiv preprint arXiv:1802.06384}, 2018.

\bibitem{yu2015lsun}
Fisher Yu, Ari Seff, Yinda Zhang, Shuran Song, Thomas Funkhouser, and Jianxiong
  Xiao.
\newblock Lsun: Construction of a large-scale image dataset using deep learning
  with humans in the loop.
\newblock {\em arXiv preprint arXiv:1506.03365}, 2015.

\bibitem{zaheer2018adaptive}
Manzil Zaheer, Sashank Reddi, Devendra Sachan, Satyen Kale, and Sanjiv Kumar.
\newblock Adaptive methods for nonconvex optimization.
\newblock In {\em Advances in Neural Information Processing Systems}, pages
  9793--9803, 2018.

\bibitem{zhang2016understanding}
Chiyuan Zhang, Samy Bengio, Moritz Hardt, Benjamin Recht, and Oriol Vinyals.
\newblock Understanding deep learning requires rethinking generalization.
\newblock {\em arXiv preprint arXiv:1611.03530}, 2016.

\end{thebibliography}
}

\section{Appendix}

\subsection{Proof of Theorems}
\label{sec_app_proof}
Assume that the data is mean normalized (mean $0$ and identity covariance matrix). In order to prove our main result, we first prove
a generic pertubation bound in the following lemma. We analyze hermite polynomials on a simpler setting, a one-hidden-layer network, to get an intuitive feeling of our argument. Consider a one-hidden-layer network with an input layer, a hidden layer, and an output layer, each with multiple units. Let $x$ denote the input vector, and $f_k(x)$ and $f_l(x)$ to be two output units. 
Denote $w_j$ to be the weight vector between input and the $j^{th}$ hidden unit,
and $a_{lk}$ to be the weight connecting $l^{th}$ hidden unit to the $k^{th}$ output unit. The network can be visualized as in Figure~\ref{fig_theorem}.
\begin{figure}[!h]
	\begin{center}
		\centerline{\includegraphics[width=0.75\columnwidth]{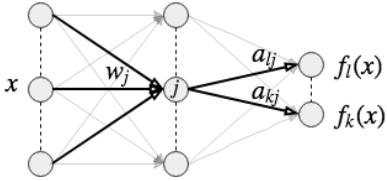}}
		\caption{One Hidden Layer Network}
		\label{fig_theorem}
	\end{center}
\end{figure}

In the lemma~\ref{lemma1}, we bound the distance between two output units $f_k(x)$ and $f_l(x)$. This gives us an upperbound on how any two output units differ as a function of the norm of the (input) data, $\|x\|$.
\begin{lemma}
	\label{lemma1}
	Consider the output unit of the network, $f_k(x) = \sum_j a_{kj} \sum_{i=0}^d c_i h_i(w_j^Tx)$, 
	where $c_i's$ are the Hermite coefficients and
	$d$ is the maximum degree of the hermite polynomial considered.
	Then, 
	$$\left|f_l(x) - f_k(x)\right| \le C d \alpha \beta$$
	Here, $C$ is a constant that depends on the coefficients of the
	Hermite polynomials and
	$$\alpha = \max_{lk} \sum_j \left| a_{lj}-a_{kj}\right| \text {  ;  } \beta = \max(\|w\|_p^d\|x\|_q^d, \|w\|_p\|x\|_q),$$
	such that $1/p+1/q=1$.
\end{lemma}
\begin{proof}
	\begin{align*}
		f_l(x) &= \sum_j a_{lj} (\sum_i c_i h_i(w_j^T x)) \text{ and } \\
		f_k(x) &= \sum_j a_{kj} (\sum_i c_i h_i(w_j^T x)) 
	\end{align*}
	\begin{align}
	|f_l(x) &- f_k(x)| \nonumber\\
	&= |\mathlarger{\sum}_j (a_{lj} - a_{kj})(\sum_i c_i h_i(w_j^T x))|\\
	\label{maineq}
	&\le |\mathlarger{\sum}_j |(a_{lj} - a_{kj})||(\sum_i c_i h_i(w_j^T x))| 
	\end{align}
	$h_i$ is the normalized hermite polynomial ($h_i = \frac{H_i}{\sqrt{i!}}$).\\
	$|h_i(z)| = |\frac{H_i(z)}{\sqrt{i!}}| \le C_1 \frac{|z|^i}{\sqrt{i!}}$, here, $C_1$ is a function of the coefficients of a given hermite polynomial.\\\\
	Let's bound $|(\sum_i c_i h_i(w_j^T x))|$,
	\begin{align*}
	|(\sum_i c_i &h_i(w_j^T x))| \\
	&\le \|c_i\|_p \|h_i\|_q \text{ where  } \frac{1}{p}+\frac{1}{q} = 1\\
	&\le \Big[\sum_i c_i^{p}\Big]^{1/p} \Big[\sum_i h_i^{q}\Big]^{1/q}\\
	&\le d^{1/p}C_2 \Bigg[\sum_{i=0}^d \frac{C_1|w_j^T x|^{iq}}{\sqrt{i!}^q}\Bigg]^{1/q}  \text{$\forall i, c_i \le C_2$}\\
	&\le d^{1/p} C_2 \Big[ \sum_{i=0}^d C_1|w_j^Tx|^{iq}\Big]^{1/q} \\
	&\text{ Replace $C_1$, $C_2$ with $C$ } \nonumber\\
	&\le d^{1/p} C d^{1/q} \max(|w_j^T x|, |w_j^T x|^d)\\
	&\le C d \max(|w_j^T x|, |w_j^T x|^d)\\
	&\text{ Using, $J = \argmax{\|w_j\|}$ and $1/p + 1/q = 1$} \nonumber\\
	&\le C d \max(\|w_J\|_{1/p}\|x\|_{1/q}, \|w_J\|_{1/p}^d \|x\|_{1/q}^d)
	\end{align*}
	Substituting the above result back into Equation~\ref{maineq}, we get,
	\begin{align*}
	&|f_l(x) - f_k(x)|\\
	&\le |\mathlarger{\sum}_j |(a_{lj} - a_{kj})||(\sum_i c_i h_i(w_j^T x))| \\
	&\le  C d \max(\|w_J\| \|x\|, \|w_J\|^d \|x\|^d) \max_{kl}\sum_j |(a_{lj} - a_{kj})|\\
	&= C d \alpha \beta
	\end{align*}
	Here, $\alpha = \max_{kl}\sum_j |(a_{lj} - a_{kj})|$ and $\beta = \max(\|w_J\| \|x\|, \|w_J\|^d \|x\|^d)$
\end{proof}
The above lemma can be seen as a perturbation bound since we can simply interpret $f_k = f_l+ g_l$ for some other function, (hopefully with nicer properties) $g_l$. 
Our main theorem is simply  an application of the above lemma to a special case. This allows us to quantify the noise resilience of hermite polynomials. In particular, we show that if the test data is far from the train data, then hermite polynomials trained deep networks give low confidence predictions. This property allows us to detect outliers during inference, especially in mission critical applications since it allows a human in a loop system \cite{yu2015lsun}.
\begin{theorem}\label{thm1}
	Let $f_k(x) = \sum_j a_{kj}\sum_{i=0}^{\infty} c_i h_i(w_j^T x)$, be a one-hidden-layer network with the sum of infinite series of hermite polynomials as an activation function. Here, $k=1,2,...,K$ are the different classes.
	Define $w_J = \min w_j^T x$. Let the data $x$ be mean normalized. If $\epsilon > 0$, the
	Hermite coefficients $c_i = (-1)^i$ and
	\begin{align*}
	\|x\| \ge \frac{1}{\|w_J\|} \log{\frac{\alpha}{\log{(1+K\epsilon)}}}
	\end{align*}
	then, we have that the predictions are approximately (uniformly) random. That is,
	\begin{align*}
	\frac{1}{K} - \epsilon \le \frac{e^{f_k(x)}}{\sum_{l=1}^K e^{f_l(x)}} \le \frac{1}{K} + \epsilon \text{  $\forall$ $k \in \{1,2...,K\}$}
	\end{align*}
	\begin{proof}
		Observe that, 
		\begin{align*}
		\frac{e^{f_k(x)}}{\sum_{l=1}^K e^{f_l(x)}} = \frac{1}{\sum_{l=1}^{K} e^{f_l(x) - f_l(x)}}
		\end{align*}
		\\
		
		Using the fact that $|x| \ge x$ and $-|x| \le x$, we observe,
		\begin{align}
		\label{eq: main}
		\frac{1}{\sum_{l=1}^{K} e^{|f_l(x) - f_k(x)|}} &\le \frac{1}{\sum_{l=1}^{K} e^{f_l(x) - f_k(x)}} \nonumber \\ &\le \frac{1}{\sum_{l=1}^{K} e^{-|f_l(x) - f_k(x)|}}
		\end{align}
		\\
		
		Let's bound $|f_l(x) - f_k(x)|$,
		\begin{align*}
		|f_l(x) &- f_k(x)| \\
		&=|\sum_j a_{kj} (\sum_i c_i h_i(w_j^T x)) - \sum_j a_{lj} (\sum_i c_i h_i(w_j^T x))| \\
		&= |\sum_j (a_{kj} - a_{lj}) (\sum_i c_i h_i(w_j^T x))|\\
		&= |\sum_j (a_{kj} - a_{lj}) (\sum_i \frac{c_i}{(-1)^n} h_i(w_j^T x)(-1)^n)|
		\end{align*}
		From the properties of hermite polynomials, 
		\begin{align}
		e^{xt - t^2/2} &= \sum_{i=0}^{\infty} h_i(x) t^n \nonumber\\
		\label{eq:sumseries}
		e^{-x - 1/2} &= \sum_{i=0}^{\infty} h_i(x) (-1)^n \text{ ,when $t = -1$}
		\end{align}
		Choose $c_i = (-1)^i$,  then
		\begin{align*}
		\max_j \sum_i (c_i)^i h_i(w_j^T x)(-1)^n &= \max_j e^{-w_j^T x - 1/2} \\
		&\le e^{- \min_j w_j^T x}
		\end{align*}
		Thus,
		\begin{align*}
		|f_l(x) &- f_k(x)| \\
		&= |\sum_j (a_{kj} - a_{lj}) (\sum_i  h_i(w_j^T x)(-1)^n)| \\
		&\le e^{- \min_j w_j^T x}|\sum_j (a_{kj} - a_{lj})|\\
		&\le  e^{- \min_j w_j^T x} \alpha \text{ where $\alpha = \max_{kl}|\sum_j (a_{kj} - a_{lj})|$}
		\end{align*}
		Let $e^{- \min_j w_j^T x} \alpha \le log(1+K\epsilon)$, then,
		\begin{align*}
		&\implies
		\frac{1}{\sum_{l=1}^{K} e^{|f_l(x) - f_k(x)|}} \le \frac{1}{\sum_{l=1}^{K} e^{f_l(x) - f_k(x)}} \\
		&\hspace{108pt}  \le \frac{1}{\sum_{l=1}^{K} e^{-|f_l(x) - f_k(x)|}}\\
		&\implies 
		\frac{1}{K(1+K\epsilon)} \le \frac{1}{\sum_{l=1}^{K} e^{f_l(x) - f_k(x)}} \le \frac{1+K\epsilon}{K}\\
		&\implies 
		\frac{1}{K} - \epsilon \le \frac{1}{\sum_{l=1}^{K} e^{f_l(x) - f_k(x)}} \le \frac{1}{K} + \epsilon
		\end{align*}
		Let $J = \argmin_j{(w_j^T x)}$. The condition, \\ $e^{- w_J^T x} \alpha \le log(1+K\epsilon)$,
		\begin{align*}
		&\implies w_J^T x \ge \log{\frac{\alpha}{\log{(1+K\epsilon)}}} \\
		&\implies \|w_J\|\|x\| \ge w_J^T x \ge \log{\frac{\alpha}{\log{(1+K\epsilon)}}} \\
		&\implies \|x\| \ge \frac{1}{\|w_J\|} \log{\frac{\alpha}{\log{(1+K\epsilon)}}}
		\end{align*}
		\end{proof}
\end{theorem}
We interpret the above theorem in the following way: whenever $\|x_{test}\|$ is large, the infinity norm of the function/prediction is approximately $1/K$ where $K$ is the number of classes. This means that the predictions are the same as that of random chance -- low confidence. The only caveat of the above theorem is that it requires the all of (infinite) hermite polynomials for the result to hold. However, {\bf all} of our experiments indicate that such a property holds true empirically, as well.

\subsection{Computational Benefits on AWS p2.xlarge}
In the paper, we have seen the cost benefits of using Hermite Polynomials on AWS p3.16xlarge instance. In this section, we will examine the performance on AWS p2.xlarge instance as indicated in Table~\ref{tab_app_awsp2x}. Clearly, as AWS p2.xlarge costs less, the benefits achieved when using hermite polynomials is more significant in the compute time.

\begin{table}[!h]
	\centering
	\resizebox{\columnwidth}{!}{%
		\begin{tabular}{c c c c c c }
			&         & Time per & Total      &                                  &   Time                                               \\ 
			&         & Epoch    & Time       & Cost                             & Savings                                          \\ 
			&         & (sec)    & (hours)    & (\$)                             & (hours)                                             \\ \hline\hline
			\multicolumn{1}{c}{\multirow{2}{*}{SVHN}}      & Hermite & 666.8   & 2.22       & \multicolumn{1}{c}{2.0}       & \multicolumn{1}{c}{\multirow{2}{*}{2.6}}      \\ 
			
			\multicolumn{1}{c}{}                           & ReLU    & 435.3    & 4.84       & \multicolumn{1}{c}{4.35}        & \multicolumn{1}{c}{}                            \\ \hline
			\multicolumn{1}{c}{\multirow{2}{*}{CIFAR-10}}   & Hermite & 454.2   & 3.53       & \multicolumn{1}{c}{3.18}          & \multicolumn{1}{c}{\multirow{2}{*}{$\ge$ 8.5}}   \\ 
			\multicolumn{1}{c}{}                           & ReLU    & 320.7   & $\ge$ 12 & \multicolumn{1}{c}{$\ge$ 10.8} & \multicolumn{1}{c}{}                            \\ \hline
			\multicolumn{1}{c}{\multirow{2}{*}{SmallNORB}} & Hermite & 164.8     & 1.74       & \multicolumn{1}{c}{1.57}       & \multicolumn{1}{c}{\multirow{2}{*}{$\ge$ 1.33}} \\ 
			\multicolumn{1}{c}{}                           & ReLU    & 81.8    & $\ge$ 3.07 & \multicolumn{1}{c}{$\ge$ 2.76}  & \multicolumn{1}{c}{}                            \\ \hline
			\multicolumn{1}{c}{\multirow{2}{*}{MNIST}}     & Hermite & 345.4    & 4.41      & \multicolumn{1}{c}{3.97} & \multicolumn{1}{c}{\multirow{2}{*}{-2}}      \\ 
			\multicolumn{1}{c}{}                           & ReLU    & 180     & 2.4       & \multicolumn{1}{c}{2.16}  & \multicolumn{1}{c}{}                            \\ \hline\hline
		\end{tabular}
	}
	\caption{\footnotesize \textbf{Hermite-SaaS on AWSp2.xlarge instance.}  Hermite-SaaS saves compute time when compared to ReLU-SaaS}
	\label{tab_app_awsp2x}
\end{table}

%

\subsection{Higher number of active units with Hermite Polynomials}
\label{sec:sparsity}
There is consensus that ReLUs suffer from the “dying ReLU” problem. Even ignoring pathological cases, it is not
uncommon to find that as much as $40\%$ of the activations
could ``die". A
direct consequence of this behavior is that the dead neurons
will stop responding to variations in error/input. This makes ReLUs
not very suitable for Recurrent Network based architectures
such as LSTMs \cite{le2015simple}. On the other hand \cite{shin2019trainability} shows that having a certain number of active units, the number determined based on the network architecture is a necessary condition for successful training. We demonstrate here that we meet  the necessary condition with a large margin. 

{\bf Experiment} We investigated the number of active units present in a Hermite Network and a ReLU network at the start and the end of training using preactResNet18 model. We determined that the percentage of active units in Hermites are  $100\% / 99.9 \%$ (EPOCH $0$/ EPOCH $200$) whereas in  ReLU  are $51\% / 45\%$ (EPOCH $0$ / EPOCH $200$). We also plot the number of active units at the end of training across the layers in Figure~\ref{fig_supp_activeunits}. It can be seen in Figure~\ref{fig_supp_activeunits} that Hermites have twice the number of Active Units than ReLU.
\begin{figure}[!h]
	\begin{center}
		\centerline{\includegraphics[width=0.75\columnwidth]{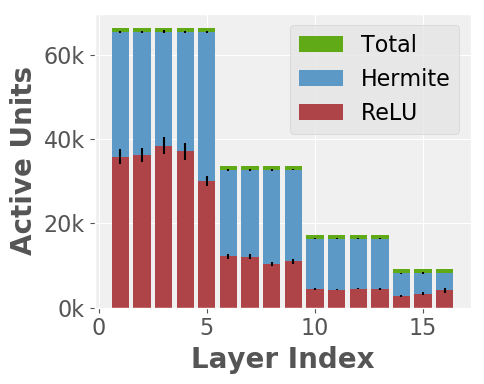}}
		\caption{\footnotesize \textbf{Hermite networks have more active units.} This figure displays the number of active units present in different layers of ResNet18 network at the end of training. It can be seen that Hermite Nets have close to $100\%$ active units while in ReLU nets half the number of neurons are dead at the end of training.}
		\label{fig_supp_activeunits}
	\end{center}
	\vskip -0.3in	
\end{figure}

\subsection{Early Riser Property Preserved: ResNets}
\label{sec_app_resnet}
{\bf Setup} We used CIFAR10 data in all the experiments in here with random crop and random horizontal flip for data augmentation schemes. With a batch size of $128$ for training and $SGD$ as an optimizer, we trained ResNets for $200$ epochs. We began with some initial learning rate and then reduce it by $10$ at $82^{nd}$ and $123^{rd}$ epoch -- this is standard practice. We repeated every experiment $4$ times in order to control the sources of variance.

{\bf Experiment} The experimental results for ResNet18 and ResNet50 models can be found in Figure~\ref{fig_app_resnet}, Figure~\ref{fig_app_resnet18} and Figure~\ref{fig_app_resnet50} respectively. The results for ResNet150 can be found in Table~\ref{tab_app_resnet152}. We observe that the early riser property is preserved accross different learning rates for these ResNet models. There is a small increase in the number of parameters proportional to the number of layers in the network. 
\begin{figure*}[!t]
	\begin{subfigure}[b]{0.5\linewidth}
		\includegraphics[width=\linewidth]{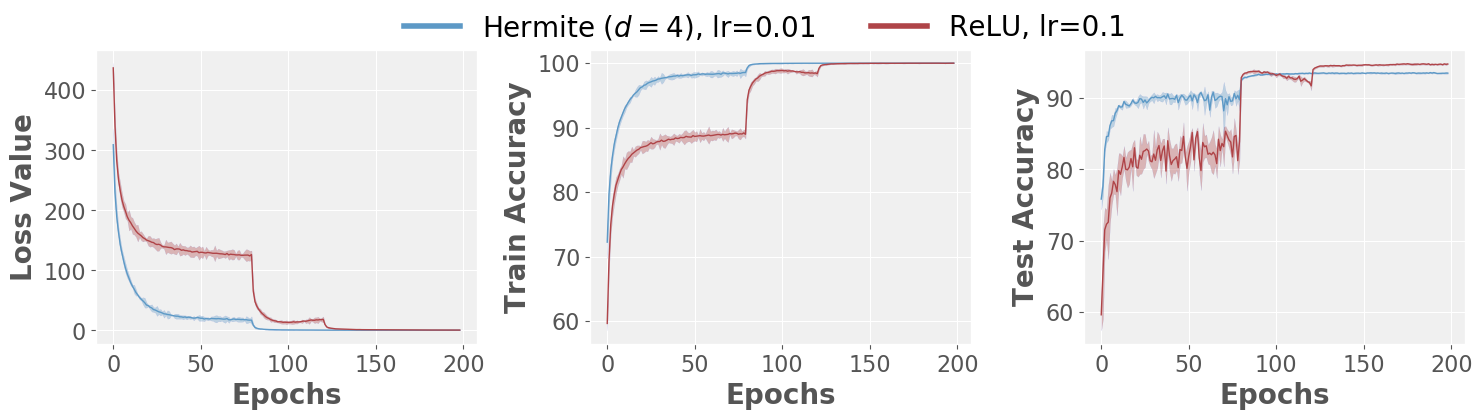}
	\end{subfigure}%
	\begin{subfigure}[b]{0.5\linewidth}
		\includegraphics[width=\linewidth]{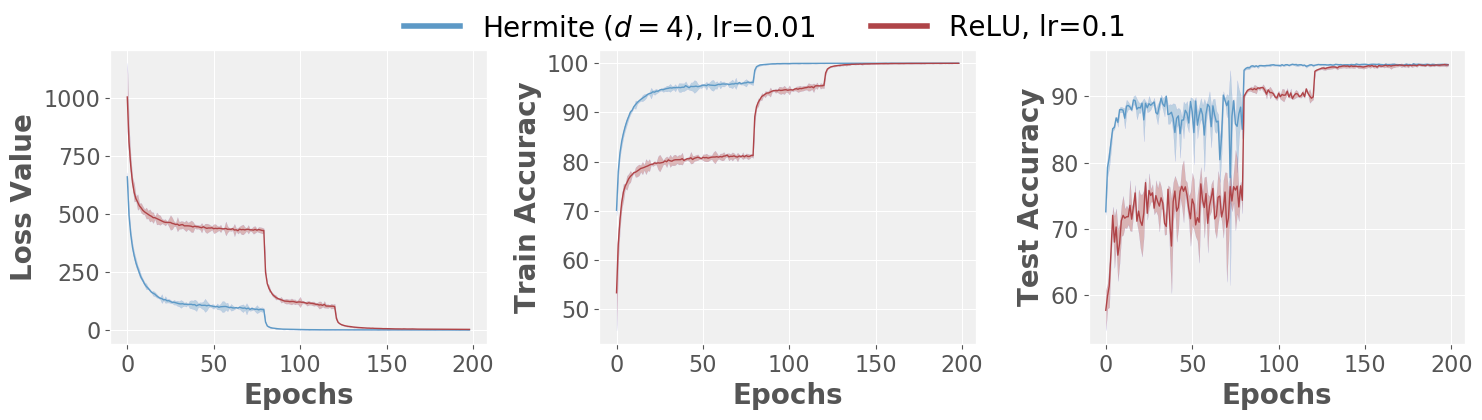}
	\end{subfigure}
	\begin{subfigure}[b]{0.5\linewidth}
		\includegraphics[width=\linewidth]{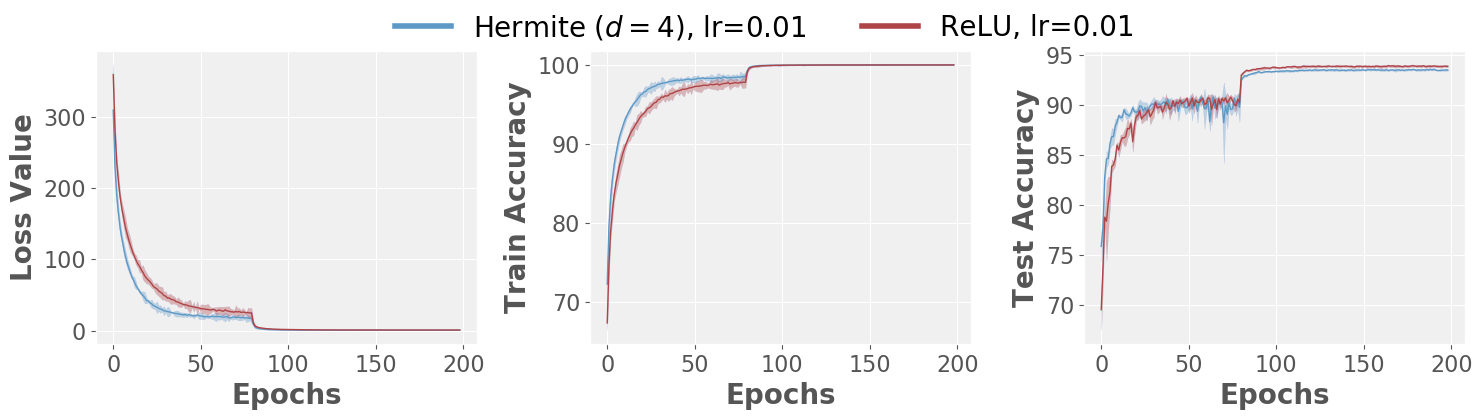}
	\end{subfigure}%
	\begin{subfigure}[b]{0.5\linewidth}
		\includegraphics[width=\linewidth]{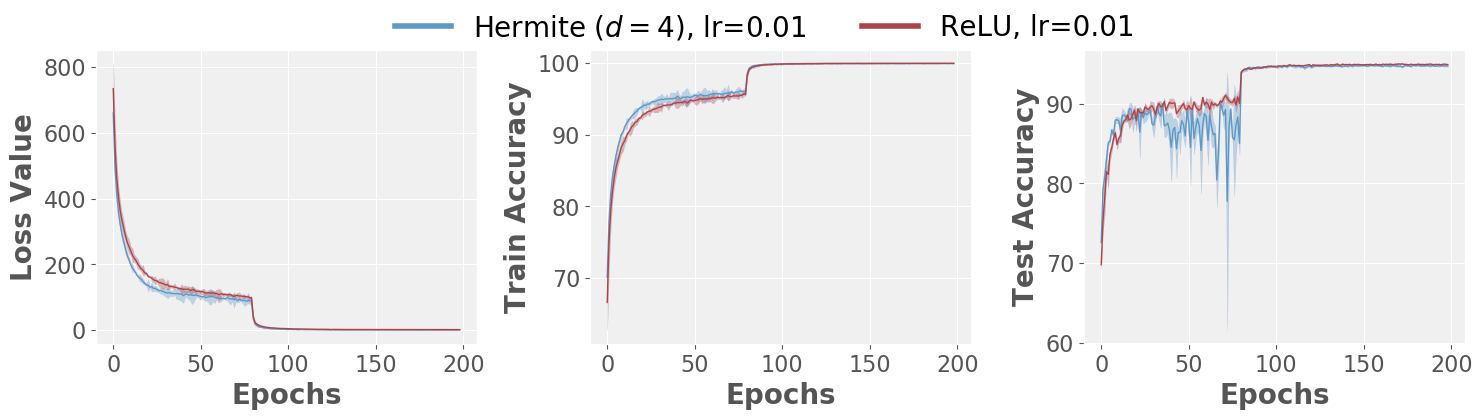}
	\end{subfigure}
	\begin{subfigure}[b]{0.5\linewidth}
		\includegraphics[width=\linewidth]{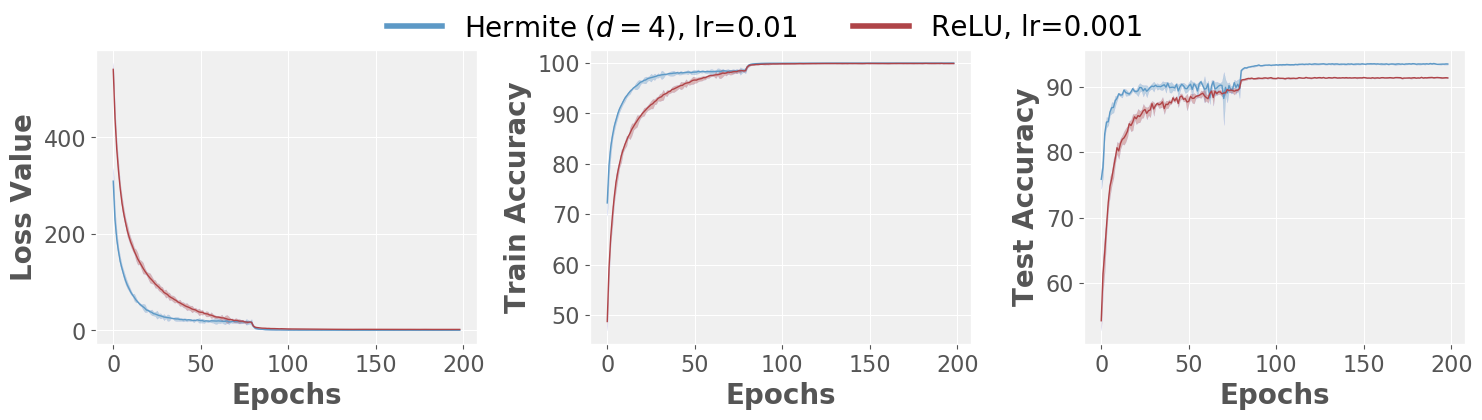}
		\caption{ResNet18}
		\label{fig_app_resnet18}
	\end{subfigure}%
	\begin{subfigure}[b]{0.5\linewidth}
		\includegraphics[width=\linewidth]{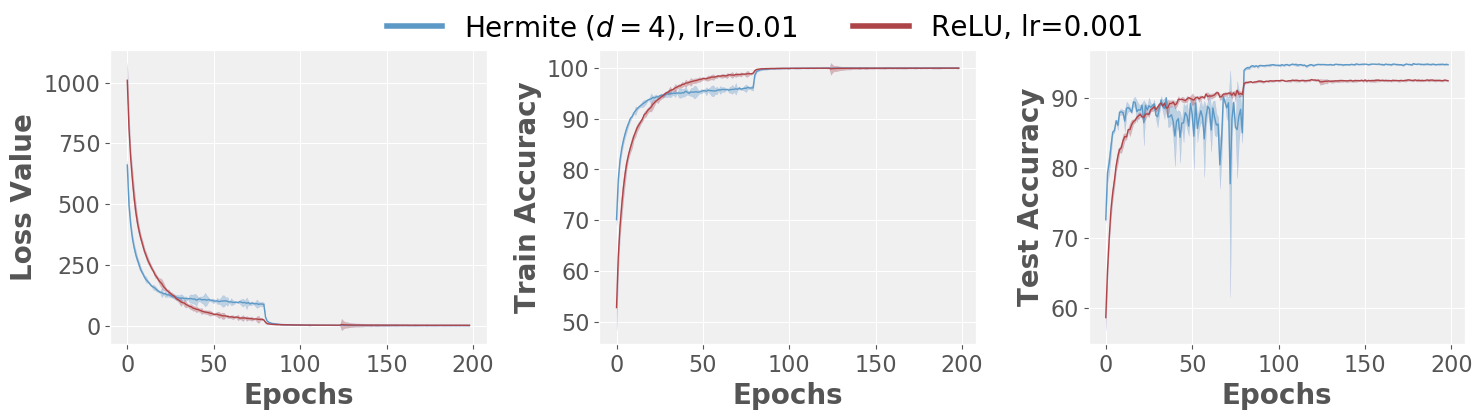}
		\caption{ResNet50}
		\label{fig_app_resnet50}
	\end{subfigure}%
	\caption{\footnotesize \textbf{Early riser property preserved in ResNets} (a) ResNet18 and (b) ResNet50. We observe that a test accuracy of $90\%$ is achieved in approximately
		half the number of epochs in
Hermite-ResNets over ReLU-ResNets on CIFAR10 dataset over different learning rates. The closest that Hermite gets to ReLU is the case when both their learning rates are $0.01$ in ResNet50. In this case, we observe that $90\%$ testset accuracy achieved in
20 epochs for Hermite-ResNet50 and 30 epochs for ReLU-ResNet50.}
	\label{fig_app_resnet}
\end{figure*} 
\begin{table}[!h]
	\centering
	\resizebox{\columnwidth}{!}{%
		\begin{tabular}{c c c c c}
			Dataset   & Number of             & Best            & Epochs to reach                 \\
			CIFAR10 & Trainable Parameters               & Test Accuracy                & $90\%$ Test Accuracy \\ \hline\hline
			\multirow{2}{*}{Hermite}      & \multirow{2}{*}{58,145,574} & \multirow{2}{*}{$95.48\%$} & \multirow{2}{*}{$30$}  \\
			&                        &                         & \\ \hline
			\multirow{2}{*}{ReLU}  & \multirow{2}{*}{58,144,842} & \multirow{2}{*}{$94.5\%$} & \multirow{2}{*}{$80$} \\             
			&                        &                         & \\	\hline
			\hline
		\end{tabular}%
	}
	\caption{\footnotesize \textbf{Hermite Polynomials in ResNet152}. We observe small increase in the number of
		parameters. Test accuracy for the hermite model converges in less than half the number of epochs.}
	\label{tab_app_resnet152}
	\vskip -0.2in
\end{table}

\subsection{Early Riser Property Preserved: DenseNets}
{\bf Setup} All the experiments here were conducted on CIFAR10 dataset with random crop and random horizontal flip schemes for data augmentation. We used the basic block architecture \cite{huang2017densely} with $40$ layers, a growth rate of $12$, and a compression rate of $1.0$ in the transition stage.

{\bf Experiment} Figure~\ref{fig_app_densenet} shows the experimental results of using Hermite Polynomials in a DenseNet. We observe that the early riser property is preserved only in the training loss and the training accuracy curves. 

\begin{figure}[!b]
	\begin{center}
		\centerline{\includegraphics[width=\columnwidth]{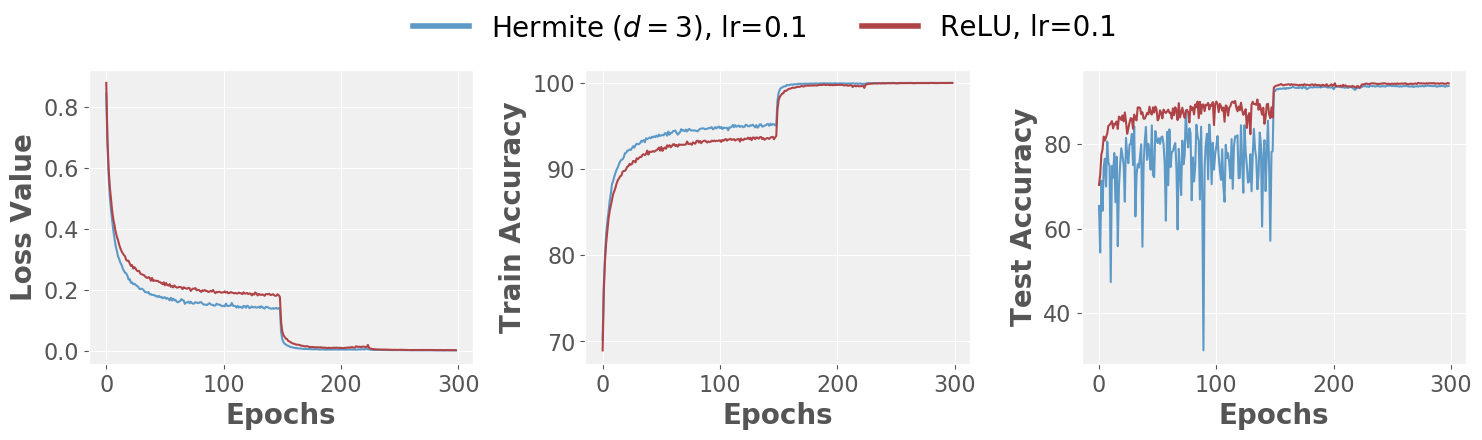}}
		\caption{\footnotesize \textbf{Early Riser Property in DenseNets.} We observe that the training loss and the training accuracy converge at a faster rate when using Hermite Polynomials in DenseNets.}
		\label{fig_app_densenet}
	\end{center}
\end{figure}

\subsection{SSL Results on CIFAR10-1K }
In the paper, we report the cost benefits on CIFAR10 dataset with 4000 labelled items. Here, we report benefits on Hermite Polynomials on CIFAR10 dataset with 1000 labelled item, another popular experimental setting for CIFAR10 dataset. We find that Hermite-SaaS works equally well for CIFAR10-1K. We achieve a max pseudolabel accuracy of $77.4\%$ 
compared to $75.6\%$ in ReLU-SaaS. 
Hermite-SaaS saves cost {\bf  $>\$280$} on AWS p3.16xlarge and compute time by 
$13+$ hours on AWS p2.xlarge, compared to ReLU-SaaS when run with $(20/135)$ inner/outer epoch config.

\subsection{Noise injection experiments on Hermites}
\label{sec_app_noiseinjection}
{\bf Setup} We repeat the SaaS experiments by injecting $10\%$ and $30\%$ label noise to CIFAR10 dataset. We utilize the method
proposed by \cite{tanaka2018joint}, let's call it \textit{Tanaka \etal},  as an optional post-processing step after
obtaining pseudo-labels via SaaS. Basically, \textit{Tanaka \etal} performs a “correction” to
minimize the influence of noise once the pseudolabels have
been estimated. The authors in \cite{tanaka2018joint} propose an alternating
optimization algorithm where the weights of the network
and the labels of the data are updated alternatively. We conduct experiments with $10\%$ and $30\%$ label noise levels on
the CIFAR10 dataset. After estimating the pseudolabels
and/or using the scheme in \cite{tanaka2018joint}, we trained a model in a
supervised manner using Resnet18.

\begin{table}[!b]
	\centering
	\resizebox{0.6\columnwidth}{!}{%
		\begin{tabular}{cccc}
			\multicolumn{2}{c}{\textbf{10\% Noise}}   & $\text{A}_{\text{Best}}$  & $\text{N}_{\text{Best}}$ \\
			\hline\hline
			\multicolumn{1}{c|}{\multirow{2}{*}{SaaS}} & Hermite & 85        & 224        \\ \cline{2-4}
			\multicolumn{1}{c|}{}                      & ReLU    & 84        & 328        \\
			\hline
			\multicolumn{1}{c|}{\multirow{2}{*}{\shortstack{SaaS + \\ \textit{Tanaka \etal}}}} & Hermite & 85  & 244 \\ \cline{2-4}
			\multicolumn{1}{c|}{}                                                    & ReLU    & 84  & 284 \\
			\hline\hline
		\end{tabular}
	}
	\vskip 0.2in
	\resizebox{0.6\columnwidth}{!}{%
		\begin{tabular}{cccccc}
			\multicolumn{2}{c}{\textbf{30\% Noise}}  & $\text{A}_{\text{Best}}$  & $\text{N}_{\text{Best}}$ \\
			\hline\hline
			\multicolumn{1}{c|}{\multirow{2}{*}{SaaS}} & Hermite & $\sim$80   & 95        \\ \cline{2-4}
			\multicolumn{1}{c|}{}                      & ReLU    & $\sim$80   & $\ge$600        \\
			\hline
			\multicolumn{1}{c|}{\multirow{2}{*}{\shortstack{SaaS + \\ \textit{Tanaka \etal}}}} & Hermite & $\sim$80 & 61  \\ \cline{2-4}
			\multicolumn{1}{c|}{}                                                    & ReLU    & $\sim$80 & 299 \\
			\hline\hline
		\end{tabular}
	}	
	\caption{\footnotesize \textbf{SaaS experiments on Noisy Labelled dataset.} We report the best accuracy ($\text{A}_{\text{Best}}$), and number of epochs ($\text{N}_{\text{Best}}$)
		to reach this accuracy. \textit{Tanaka \etal} stands for noisy label processing method proposed in \cite{tanaka2018joint}.}
	\label{tab:noisy_saas}
\end{table}

{\bf Experiment} Our results summarized in Table~\ref{tab:noisy_saas} show that Hermite-SaaS based models obtain a similar or higher test set accuracy. We also observe that our
model converges faster compared to a ReLU-SaaS model.
Our experimental results also indicate that post processing techniques (such as \cite{tanaka2018joint}) may not always be usefuls
to improve the generalization performance of models.

\subsection{Experiments on Shallow Nets}
\label{sec_app_shallow}
{\bf Setup} We used a 3-layer network where the hidden layers have 256 nodes each. We used CIFAR10 dataset for this experiment with normalized pixel values and no other data augmentation schemes. The architecture we worked with is thus [3072, 256, 256, 10]. We ran two experiments, one with ReLU activation function and the other with Hermite activations ($d=5$ polynomials). The loss function was cross-entropy, SGD was the optimization algorithm and we added batch normalization. Learning rate was chosen as $0.1$ and the batch size as $128$. Figure~\ref{fig_app_shallownet} shows the loss function and the training accuracy curves that we obtained. It was observed that, when ReLU is just replaced with Hermite activations while maintaining everything else the same, the loss function for hermites converges at a faster rate than ReLU atleast for the earlier epochs. 

\begin{figure}[!h]
	\begin{center}
		\centerline{\includegraphics[width=\columnwidth]{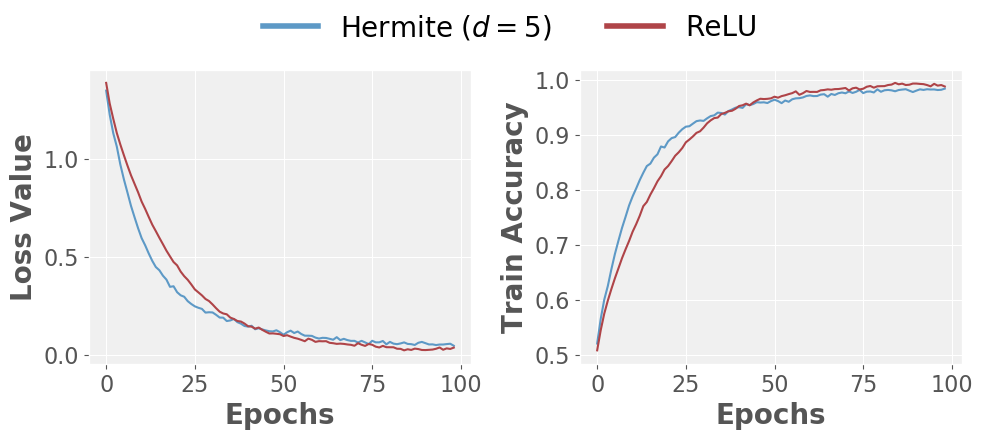}}
		\caption{\footnotesize\textbf{Experiments on Shallow Nets.} The training loss and training accuracy reduce at a faster rate with Hermite-Polynomials}
		\label{fig_app_shallownet}
	\end{center}
	\vskip -0.3in
\end{figure}

\subsection{Cloud Services Details}
We utilize AWS EC2 On-Demand instances in our research. Specifically, we use p3.16xlarge instance which costs $\$ 24.48$ per hour and p2.xlarge instance which costs $\$ 0.9$ per hour.

\end{document}